\pdfoutput=1
\documentclass[10pt]{article}
\usepackage[margin=1in]{geometry}
\usepackage{yub}
\usepackage[compress,numbers]{natbib}
\usepackage{parskip}
\usepackage{enumitem}
\usepackage{tikz}
\usepackage{hyperref}       %
\usepackage{verbatim}
\usepackage{algorithm}
\usepackage[noend]{algorithmic}
\usepackage{subcaption}
\usepackage{graphicx}

\allowdisplaybreaks

\usepackage{makecell}
\definecolor{light-gray}{gray}{0.85}
\usepackage{colortbl}

\newcommand{\cO}{\mc{O}}

\newcommand{\tO}{\wt{\mc{O}}}

\renewcommand{\epsilon}{\eps}

\renewcommand{\hat}{\widehat}

\renewcommand{\bar}{\overline}

\def\cX{{\mathcal X}}

\def\cF{{\mathcal F}}

\def\id{\mathbf{I}}

\mathchardef\mhyphen="2D

\newenvironment{talign*}
 {\csname align*\endcsname}
 {\endalign}

\newcommand{\loss}[1]{\ell\paren{#1}}
\newcommand{\populoss}[1]{L_D(#1)}

\newcommand{\sequenceX}[1]{\bx^{#1}_{0:N}}

\newcommand{\tfwm}{{f}_{M}^{\bW}}
\newcommand{\tfwmb}{{f}_{M}^{\bW,\bW_0}}
\newcommand{\fnclass}{\mathcal{V}_M}
\newcommand{\dn}{\frac{1}{n}}
\newcommand{\sumn}{\sum_{j=1}^n}
\newcommand{\rkhsbound}{B\paren{f_\star}} 
\newcommand{\biasbound}{B(g_\star, L)}
\newcommand{\tflp}{f_M^{\textsc{MLP}}}
\newcommand{\cV}{{\mathcal V}}
\newcommand{\ltwob}[1]{\Big\|#1\Big\|_2}
\newcommand{\mlpbound}{B_{\rm MLP}(f_\star)}

\newcommand{\feasy}{{\wt{g}_{\star}}}
\newcommand{\fhard}{{g_{\star}}}
\newcommand{\Fr}{{\sf Fr}}
\newcommand{\RF}{{\tt RF}}
\newcommand{\RFA}{{\tt RFA}}
\newcommand{\RFMLP}{{\tt RFMLP}}
\newcommand{\BRFA}{{\tt BRFA}}
\newcommand{\Poly}{{\rm Poly}}
\newcommand{\Unif}{{\sf Unif}}
\newcommand{\mbf}{{\mathbf{f}}}
\newcommand{\bG}{{\mathbf{G}}}

\newcommand{\sfP}{{\mathsf P}}

\newcommand{\Bvone}{{K_{1}}}
\newcommand{\Bvtwo}{{K_{2}}}

\newcommand{\de}{{\rm d}}
\newcommand{\cD}{{\mathcal D}}
\newcommand{\cG}{{\mathcal G}}

\def\vm{{\bm{m}}}

\def\bA{{\mathbf A}}

\def\bD{{\mathbf D}}

\def\bG{{\mathbf G}}

\def\bI{{\mathbf I}}
\def\bK{{\mathbf K}}

\def\bQ{{\mathbf Q}}
\def\bS{{\mathbf S}}

\def\bV{{\mathbf V}}
\def\bW{{\mathbf W}}

\def\bbeta{{\boldsymbol \beta}}
\def\bsigma{{\boldsymbol \sigma}}

\def\be{{\mathbf e}}

\def\bk{{\mathbf k}}
\def\bq{{\mathbf q}}

\def\bu{{\mathbf u}}
\def\bv{{\mathbf v}}
\def\bw{{\mathbf w}}
\def\bx{{\mathbf x}}
\def\by{{\mathbf y}}

\newcommand{\barsig}{\sigma}

\newcommand{\bzero}{{\mathbf 0}}

\def\bX{{\mathbf X}}
\def\tbX{\tbx_0\tbx_{i}^\top}
\def\tbXp{\tbx_0^\prime(\tbx_{j}^\prime)^\top}
\def\tbx{{\wt{\mathbf x}}}
\def\bId{{\bI_{d\times d}}}
\def\fwz{f^{\bW_0}}
\def\He{{\rm He}}
\def\Px{{P}}

\hypersetup{
    colorlinks,
    linkcolor={blue!70!black},
    citecolor={blue!70!black},
}
\colorlet{linkequation}{blue}

\usepackage{cleveref}
\crefformat{equation}{#2(#1)#3}
\Crefformat{equation}{#2(#1)#3}

\def\shownotes{1}  %
\ifnum\shownotes=1
\newcommand{\authnote}[2]{{\scriptsize $\ll$\textsf{#1 notes: #2}$\gg$}}
\else
\newcommand{\authnote}[2]{}
\fi
\newcommand{\yub}[1]{{\color{red}\authnote{Yu}{#1}}}

\newcommand{\hengyu}[1]{{\color{green}\authnote{Hengyu}{#1}}}

\title{What Can a Single Attention Layer Learn? A Study Through the Random Features Lens}

\author{
Hengyu Fu\thanks{Equal contributions.}\hspace{.35em}\thanks{Peking University. Email: \texttt{2100010881@stu.pku.edu.cn}}
\and
Tianyu Guo\footnotemark[1]\hspace{.35em}\thanks{UC Berkeley. Email: \texttt{\{tianyu\_guo,songmei\}@berkeley.edu}}
\and
Yu Bai\thanks{Salesforce AI Research. Email: \texttt{yu.bai@salesforce.com}}
\and
Song Mei\footnotemark[3]
}

\begin{document}

\maketitle

\begin{abstract}

Attention layers---which map a sequence of inputs to a sequence of outputs---are core building blocks of the Transformer architecture which has achieved significant breakthroughs in modern artificial intelligence. This paper presents a rigorous theoretical study on the learning and generalization of a single multi-head attention layer, with a sequence of key vectors and a separate query vector as input. We consider the \emph{random feature} setting where the attention layer has a large number of heads, with randomly sampled frozen query and key matrices, and trainable value matrices. We show that such a random-feature attention layer can express a broad class of target functions that are permutation invariant to the key vectors.  We further provide quantitative excess risk bounds for learning these target functions from finite samples, using random feature attention with finitely many heads.

Our results feature several implications unique to the attention structure compared with existing random features theory for neural networks, such as (1) Advantages in the sample complexity over standard two-layer random-feature networks;  (2) Concrete and natural classes of functions that can be learned efficiently by a random-feature attention layer; and (3) The effect of the sampling distribution of the \emph{query-key} weight matrix (the product of the query and key matrix), where Gaussian random weights with a non-zero mean result in better sample complexities over the zero-mean counterpart for learning certain natural target functions. Experiments on simulated data corroborate our theoretical findings and further illustrate the interplay between the sample size and the complexity of the target function. 
\end{abstract}

\section{Introduction}

The transformer architecture~\citep{vaswani2017attention} has achieved remarkable recent successes in many areas of artificial intelligence (AI) such as vision, language, speech, graph processing, reinforcement learning, and more recently general AI capabilities~\citep{devlin2018bert,dong2018speech,dosovitskiy2020image,brown2020language,radford2021learning,ying2021transformers,chen2021decision,reed2022generalist,openai2023gpt,bubeck2023sparks}. A central building block in transformers is the \emph{attention layers}~\citep{bahdanau2014neural}---sequence-to-sequence mappings that allow each token within the input sequence to ``attend to'' other tokens that are most relevant to the present token, and produce outputs based on those tokens. Attention layers implement this mechanism in a compact way that allows them to handle sequences of arbitrary length using a fixed set of parameters, a crucial reason behind their success in handling long input sequences.

Despite its wide applicability, the theoretical properties of attention layers are less well understood. While multi-layer attention networks (transformers) have been shown to be universal approximators for certain classes of functions, such as equivariant sequence-to-sequence functions~\citep{yun2019transformers}, their results only focus on the expressive power and do not account for learning from finite samples. Another line of work derives generalization bounds for learning with \emph{multi-layer} transformers in terms of the number of layers, heads, and weight norms~\citep{wei2022statistically,edelman2022inductive}, yet the results are either instantiated on specific target functions such as sparse boolean functions~\citep{edelman2022inductive}, or generic but arguably elusive function classes such as Turing machines~\citep{wei2022statistically}. Understandings about the more basic building block---a \emph{single} attention layer--- remain largely open. This is in stark contrast with the situation for fully connected neural networks, where there is by now a decent understanding of the learning and generalization in the important basic case of \emph{two-layer} neural networks on generic and natural function classes (e.g.,~\citep{bach2017breaking,li2018learning,arora2019fine,ghorbani2021linearized} and many other results along the line). This motivates the following open question: 
\begin{center}
    \emph{What function classes can be learned by a \textbf{single attention layer} with benign sample complexities?}
\end{center}

This work makes progress on this problem by studying the learning and generalization with a single attention layer in the \emph{random feature} setting~\citep{rahimi2007random,rahimi2008weighted,daniely2017sgd,yehudai2019power}, in which the query and key matrices are frozen at their random initialization, and the value matrices remain to be learnable parameters. Motivated by the attention structure in practical architectures, we consider attention layers that take in a single \emph{query token} $\bx_0\in\R^d$ and $N$ \emph{key tokens} $\sets{\bx_i}_{i\in[N]}$ as the input, and produce a scalar-valued output---A simplified setting capturing the essence (the interaction between the query and keys) of attention models.
We study the sample complexity of learning certain target functions (of $\bx_{0:N}$) using an attention layer with a large but finite number of heads, and finitely many samples. %

Our contributions are summarized as follows.
\begin{itemize}[leftmargin=2em,topsep=0pt]
    \item We show that a Random Feature Attention layer (the \RFA{} model) with a sufficiently large number of heads can express a broad class of target functions that are averages over a generic function of two tokens, which are in particular permutation invariant with respect to the key tokens (Section~\ref{sec:expressivity of rf}). We give several natural examples of target functions in this class (Section~\ref{sec::examples and comparison}) with concrete bounds on the number of heads and weight norms.
    \item We derive an $\tO(\sqrt{\rkhsbound/n})$ excess risk bound for learning with the \RFA{} model with sufficiently many heads, where $\rkhsbound$ is an inherent complexity measure of the target function $f_\star$ and $n$ is the sample size (Section~\ref{section:sample complexity}). When instantiated on the aforementioned examples, the bounds only depend on the input dimension and not the number of key tokens, improving over a naive two-layer random feature neural network model (\RFMLP{}). Such improvement is expected due to the permutation invariance structure of target functions, aligning with the attention mechanism. 
    \item Towards moving beyond standard random feature settings, we study a \emph{biased} \RFA{} model where the query-key matrices (product of transposed query matrices and key matrices) are drawn from a distribution with a non-zero mean (more precisely the identity matrix as the mean), motivated by a similar observation on learned attention layers in practice. We show that this model achieves provably superior sample complexity than the standard zero-mean \RFA{} for learning certain functions of the \emph{correlations} between the query token and key tokens (Section~\ref{section:bias attention}).
    \item Experiments on simulated data verify our theoretical findings in realistic settings of learning from finite samples using a \RFA{} layer with a mild number of heads, and characterize the interplay between the complexity of the target function and the sample size (Section~\ref{sec:experiments}).
\end{itemize}

\subsection{Related work}

\paragraph{Transformers } The Transformer architecture, initially proposed by \cite{vaswani2017attention}, brought about a revolutionary change in natural language processing and has been widely adopted in large language models such as GPT and BERT \cite{radford2018improving, devlin2018bert, brown2020language}. At the core of transformers lie the \emph{attention layers}, which were originally introduced as neural network modules for machine translation tasks \citep{bahdanau2014neural, kim2017structured, parikh2016decomposable}.  

A recent line of work investigated the capabilities of transformers by viewing transformers to be function approximators \cite{yun2019transformers} or computational models \cite{wei2022statistically, perez2019turing, yao2021self, bhattamishra2020computational, liu2022transformers}, and using transformers to perform synthetic reasoning tasks \cite{zhang2022unveiling}. Among these works, the closest to our work is \cite{yun2019transformers}, which shows that multi-layer transformers can approximate any permutation-equivariant sequence-to-sequence function, and any function if positional encodings are added.  Our paper instead uses a single attention layer to approximate sequence-to-scaler functions and focuses on the generalization property with quantitative bounds.

In terms of generalization properties of transformers, \cite{edelman2022inductive} analyzed the generalization bound of a single attention network through the Rademacher complexity and showed that a single self-attention head could efficiently represent a sparse function of the input sequence. Our paper also studies the generalization bound of a single attention network, but from the different perspective of kernel methods. The kernel limit of transformers was derived in \cite{hron2020infinite, yang2020tensor}, which shows that multi-head attention architectures behave as Gaussian processes as the number of heads tends to infinity. However, they do not study the representation power of the limiting kernel. 

Besides approximation and generalization capabilities, recent work also studied the limitations \cite{hahn2020theoretical, bhattamishra2020ability}, internal working mechanisms \cite{elhage2021mathematical, snell2021approximating, weiss2021thinking, olsson2022context}, and in-context learning capabilities \cite{brown2020language, xie2021explanation, garg2022can, von2022transformers, akyurek2022learning, dai2022can, giannou2023looped, li2023transformers} of Transformer models.

\paragraph{Theorey of random features and neural tangent kernels}
A recent line of work \cite{daniely2017sgd, li2018learning, du2018gradient, du2019gradient, allen2019convergence, arora2019fine, zou2020gradient,oymak2020toward, chizat2019lazy} studied the training dynamics of overparametrized neural networks under certain random initialization, and showed that it converges to a kernel estimator, which corresponds to the “neural tangent kernel” (NTK). These works suggested that one could use kernel or random-feature models \cite{rahimi2007random} to study the properties of deep neural networks. 

For NTK of MLPs and their corresponding random-feature models, there is a vast number of literature that studies their approximation power \cite{barron1993universal, pinkus1999approximation, bach2017breaking}, as well as their generalization properties \cite{bartlett2002rademacher, caponnetto2007optimal, weinan2019comparative, weinan2019barron, liang2020just, liang2019risk,rahimi2008weighted,rudi2017generalization, yehudai2019power,mei2022generalization,mei2022generalization2, ghorbani2021linearized, ghorbani2020neural}. More recently, a line of work studies the NTK beyond MLPs, including convolution networks \cite{li2019enhanced, bietti2019inductive, mei2021learning, misiakiewicz2021learning, bietti2021sample, bietti2022approximation}, residual networks \cite{huang2020deep, tirer2022kernel, allen2019can}, graph networks \cite{xu2020neural, jegelka2022theory}, and transformers \cite{hron2020infinite, yang2020tensor}. 

Although the kernel approach is a powerful tool for studying neural networks, it received criticism since it does not capture the feature learning of neural networks. Going beyond the kernel regime, a series of works used the mean field method to establish the evolution of the network parameters via a Wasserstein gradient flow \cite{mei2018mean, bach2021gradient, chizat2018global, rotskoff2018neural}. 
Several other mechanisms have been proven to obtain superior results over the NTK, including the Quadratic NTK~\citep{allen2019learning,bai2019beyond,chen2020towards,nichani2022identifying}, regularization~\citep{wei2019regularization}, Neural Tangent Hierarchy~\cite{dyer2019asymptotics,huang2020dynamics}, representation learning~\citep{damian2022neural}, and staircase-like mechanisms~\citep{abbe2022merged,abbe2023sgd}. %

\section{Preliminaries}

\newcommand{\Qm}{\bQ_m}
\newcommand{\Km}{\bK_m}
\renewcommand{\vm}{\bv_m}

We consider a sequence of $N + 1$ input tokens $\sequenceX{} = (\bx_0, \{\bx_i \}_{i \in [N]}) \in \cX = (\R^d)^{N+1}$, where each $\{ \bx_i \}_{i \in [N]} \subseteq \R^d$ represents a sequence of \emph{key vectors}, and $\bx_0$ represents the \emph{query vector}. Throughout the paper, we consider scalar-valued attention models, which take a sequence as input and give a scalar output in $\R$. This model simplifies---but preserves the essence of---a full sequence-to-sequence self-attention on $\bx_{1:N}$, by taking the query token to be $\bx_0$ instead of one of $\bx_i$, $i\in[N]$.

\paragraph{Attention layer}
We consider a scalar-valued, $M$-head, multiplicative  attention layer that takes $\sequenceX{} = (\bx_0, \{\bx_i \}_{i \in [N]})$ as the input. The attention layer first applies affine (linear with bias) transformations to the input vectors to obtain \{query, keys, values\} at each head $m\in[M]$:
\begin{equation}\label{eqn:qkv-vectors}
\begin{aligned}
    & \bq_{m,0} = \Qm [\bx_0; 1] \eqdef \Qm \tbx_0 \in \R^d, ~~~~~~~~~\bk_{m,i} = \Km [\bx_i; 1] \eqdef \Km \tbx_i \in \R^d, \\
    & v_{m,i} = \vm^\top [\bx_i; 1] = \vm^\top \tbx_i \in \R,~~~i\in[N],
\end{aligned}
\end{equation}
where $\Qm,\Km\in\R^{(d+1)\times d}$, $\vm\in\R^{d+1}$ are the parameters of the attention layer, and $\tbx_i\defeq [\bx_i; 1]$ for a more compact display. Then, it computes the output value by an attention mechanism %
\begin{align}
\label{eqn:attention}
f(\sequenceX{}) =  \sum_{m=1}^M \frac{1}{N} 
 \sum_{i = 1}^N f_{m,i}(\bx_0, \bx_i), ~~~~~~ f_{m, i}(\bx_0, \bx_i) = \sigma(\<\bq_{m,0}, \bk_{m,i}\>)\cdot v_{m,i} \in \R.
\end{align}
Above, $\sigma:\R\to\R$ is an activation function applied entry-wisely to each attention score $\<\bq_{m,0}, \bk_{m,i}\>$.  We choose $\sigma$ to be the ReLU activation $\sigma(t)=\max\sets{t,0}$ throughout this paper. Notice that this choice of the attention non-linearity is different from standard transformers \cite{vaswani2017attention} with softmax-attention. We choose to study ReLU-attention for theoretical convenience, and this replacement does not change the essence of the attention mechanism. Such a choice is also recently explored in the literature \cite{shen2023study}, which shows that transformers with ReLU-attention perform as well as standard softmax-attention transformers in certain NLP tasks. %

Simplifying the expression, we reparametrize the attention layer \eqref{eqn:qkv-vectors} and \eqref{eqn:attention} using parameters $\sets{(\bW_m, \bv_m)}_{m\in[M]}\subseteq \R^{(d+1)\times(d+1)}\times \R^{d+1}$:
\begin{align}
f_{i, m}(\sequenceX{}) =&~ \sigma\Big( \tbx_0^\top \Qm\Km \tbx_i \Big) \cdot \<\bv_m, \tbx_i\> 
= \sigma\paren{ \<\bW_m, \tbx_0 \tbx_i ^\top\> } \cdot \<\bv_m, \tbx_i\>. \label{eqn:attention-reparam}
\end{align}
For technical convenience, we assume all input tokens have unit norm throughout the rest of the paper: $\ltwos{\bx_i}\equiv 1$ so that $\ltwos{\tbx_i}\equiv \sqrt{2}$, for all $i\in\sets{0} \cup [N]$.

\paragraph{Random-feature attention models}
We consider a random-feature version\footnote{A different and closely related model is the Neural Tangent Kernel~\citep{jacot2018neural,du2019gradient}, which is however similar in essence to the random feature model in terms of the sample complexity of learning, e.g.~\citep{ghorbani2021linearized}.
} of the multiplicative attention mechanism~\eqref{eqn:attention-reparam}, where the weight matrices $\sets{\bW_m}_{m\in[M]}$ have i.i.d. Gaussian entries\footnote{
Another feasible choice for~\eqref{eqn:w-distribution} is to sample the key matrix and the query matrix separately with independent Gaussian entries, which however will produce a mean-zero product matrix similar to~\eqref{eqn:w-distribution} in many aspects.
}:
\begin{align}
\label{eqn:w-distribution}
    (\bW_m)_{ij} \simiid \normal(0, 1/4),~~~(m,i,j)\in[M]\times[d+1]^2.
\end{align}
The variance is chosen to be $1/4$ without loss of generality: this choice of variance is such that $\langle\bW_m, \tbx_0 \tbx_i ^\top \rangle \sim\normal(0,1)$ has a unit variance. The weight matrices $\sets{\bW_m}_{m \in [M]}$ are then held to be fixed during the entire learning process, whereas the value vectors $\sets{\bv_m}_{m\in[M]}$ are the learnable parameters. %
The random-feature attention model with input $\sequenceX{}$ %
is thus given by
\begin{align}
\textstyle    \tfwm(\sequenceX{}; \bV) = \sum_{m=1}^M \frac{1}{N}\sum_{i=1}^N \sigma\big(\<\bW_m, \tbx_0 \tbx_i ^\top\>\big) \<\bv_m, \tbx_i\>. \label{eqn:RF_attention}
\end{align}
Notice that random-feature attention model is linear in the parameter $\bV$, so training this model with a convex loss function gives a convex optimization problem.

\paragraph{Additional notation}
For any $\bx\in\R^{d_1}$ and $\by\in\R^{d_2}$, let $\bx\otimes \by\in\R^{d_1\times d_2}$ denote their tensor product (outer product), and $\bx^{\otimes n}\defeq \bx\otimes \dots \otimes \bx$ denote the $n$-fold self tensor product of $\bx$. For a tensor $\bA$, we use $\lfro{\bA}$ to denote its Frobenius norm. For a function $f: \cX \to \R$, we use $\| f \|_{\infty}$ to denote its $L^\infty$ norm. We use $\cO(\cdot)$ (resp. $\Theta(\cdot)$) for standard Big-O (resp. Big-Theta) relations. We use $\tO(\cdot)$ for hiding the multiplicative terms that are logarithmic in problem parameters, including $(M, d, n, N, \delta^{-1})$. We use $\Poly(p)$ to denote a polynomial of $p$ that is less than $p^C$ for some universal constant $0 < C < \infty$.

\section{Learning with random-feature attention models}\label{section:RKHS analysis}

In this section, we study the expressivity and generalization of random-feature attention models. We will consider a broad class of target functions that can be well approximated and is efficiently learnable by random-feature attention models. 

\subsection{Expressivity of random-feature attention}
\label{sec:expressivity of rf}

Consider a broad class of target functions $f_\star:\cX\to\R$ that takes form
\begin{align}
\label{eqn:target-function}
\textstyle    f_\star(\sequenceX{}) = \frac{1}{N} \sum_{i=1}^N F(\bx_0,\bx_i). 
\end{align}
Assume that there exists symmetric tensors $\{ \mathbf{f}_{rs} \in \R^{d^{r + s}}\}_{r, s\ge 0}$ such that $F: \R^{2d} \to \R$ admits representation
\begin{equation}\label{eqn:target-expansion}
\textstyle F(\bx_0,\bx_i) = \sum_{r,s \ge 0}^\infty \<\bx_0^{\otimes r}\otimes \bx_i^{\otimes s}, \mathbf{f}_{rs}\>.
\end{equation}

Note that such an expression allows $F(\bx_0,\bx_i)$ to be any general nonlinear function that admits convergent Taylor expansions. In particular, any polynomials of $[\bx_0, \bx_i]$ (e.g., $\bbeta^\top \bx_0$, $\bbeta^\top\bx_i$, and $\<\bx_0,\bS\bx_i\>$ for some $\bbeta \in \R^d$ and $\bS \in \R^{d^2}$) are within this function class. We will discuss more specific target functions in Section \ref{sec::examples and comparison}. %

\begin{theorem}[Expressivity of \RFA~model]
\label{thm:finite-width}

Suppose function $f_\star:\cX\to\R$ takes form~\eqref{eqn:target-function}. Then for any input distribution $\Px$ on $\cX$, with probability at least $1-\delta$ (over $\sets{\bW_m}_{m\in[M]}$ sampled from~\eqref{eqn:w-distribution}), there exists an $M$-head \RFA~model (\ref{eqn:RF_attention}) with coefficients $\bV = \sets{\bv_m}_{m\in[M]}\subseteq \R^{d+1}$ that approximates $f_\star$ in $L^2(\Px)$ up to error
\begin{align}
    \E_{\bx_{0:N}\sim \Px}\brac{ \paren{ f_\star(\bx_{0:N}) -\tfwm(\bx_{0:N}; \bV) }^2 } \le \cO\bigg(\frac{ (d^2 + \log M) \rkhsbound \delta^{-1}}{ M} \bigg). \label{eqn:approx without Id}
\end{align}
In addition, the norms of the weight of this random-feature attention model are bounded as
\begin{align}
    \sum_{m=1}^M \ltwo{\bv_m} \le \cO \bigg( \sqrt{\rkhsbound} + \sqrt{\frac{ \rkhsbound \delta^{-1}}{M}} \bigg),~~~~~~ \sum_{m=1}^M \ltwo{\bv_m}^2 \le \cO \bigg( \frac{ \rkhsbound \delta^{-1}}{M} \bigg). \label{eqn:norm bound of v without Id}
\end{align}
Here $\rkhsbound$ is a complexity measure of $f_\star$ defined as 
\begin{align}
\label{eqn:bfstar}
\textstyle \rkhsbound = \sum_{k=0}^\infty C_{k} \sum_{\max\{r,s\}=k} \lfro{\mbf_{rs}}^2, ~~~~~~~~~ C_k = k^{4.5}4^k\vee 1.
\end{align}
In case where $f_\star$ admits multiple representations of the form~\eqref{eqn:target-expansion}, $\rkhsbound$ is the infimum of the right-hand-side over all such representations. 
\end{theorem}
The proof of Theorem \ref{thm:finite-width} is contained in Appendix \ref{sec:proof-finite-width}. Our proof relies on standard analyses of infinite-width random feature model with ReLU-Gaussian kernel, combined with a sampling argument to obtain approximation with finite-width.

This theorem is applicable to general functions with a finite $\rkhsbound$ norm. The $4^k$ scaling of $C_k$ in the summand of equation (\ref{eqn:bfstar}) seemingly confines the target function class to those with exponentially fast decaying $\| \mbf_{rs} \|_{\Fr}$, which suggests a relatively narrow target function class. However, as we will demonstrate in the forthcoming examples, this class includes a diverse range of functions.

\subsection{Generalization and sample complexity of learning}
\label{section:sample complexity}

Given $n$ samples $\sets{\sequenceX{(j)},y_j}_{j\in[n]} \simiid \sfP$, where $\sequenceX{(j)} = \{ \bx_i^{(j)} \}_{0 \le i \le N}$ is the $j$-th token sequence with length $N+1$, and $y_j$ is the label corresponding to the $i$-th token sequence. Assume that we are given a loss function $\ell\paren{\hat y,y}$ that is 1-Lipschitz in $\hat y$, and $\ell(0,y)\le 1$ for any $y$. The population risk is then given by $\populoss{f} = \E_{(\sequenceX{}, y) \sim \sfP}\brac{\loss{f(\sequenceX{}),y}}$. We consider the empirical risk minimization (ERM) over the \RFA~model \eqref{eqn:RF_attention}, 
\begin{equation}
\textstyle \hat\bV = \argmin_{\bV \in \mathcal{V}_M} \hat{L}_D(f_{M}^{\bW}(\cdot;\bV)), ~~~~~~~~~ \hat{L}_D(f) = \dn \sum_{j=1}^n \ell(f(\sequenceX{(j)}),y_j), \label{equ::ERM fomula}
\end{equation}
where the constrained class $\mathcal{V}_M$ is given by
\begin{equation}
\textstyle \cV_M = \left\{\bV = \set{\bv_m}_{m=1}^M : \,\, \sum_{m=1}^M \ltwo{\bv_m} \le \Bvone , \sum_{m=1}^M \ltwo{\bv_m}^2 \le \Bvtwo /M\right\}, \label{equ::ERM constraint}
\end{equation}
with $\Bvone$ and $\Bvtwo$ being two constants. 
Theorem \ref{thm::sample complexity} below provides the excess risk bound for the empirical risk minimizer.  
\begin{theorem}
\label{thm::sample complexity}
Assume $M > \delta^{-1}$ and $n > \log(dM)$. 
Let $f_\star$ be the minimizer of the population risk $\populoss{f}$ within the target function class \eqref{eqn:target-function} \eqref{eqn:target-expansion}. Let $\hat{f}^{\bW}_{M} = \tfwm(\cdot;\hat\bV)$ be the empirical risk minimizer given by \eqref{equ::ERM fomula}, where in \eqref{equ::ERM constraint} we choose $\Bvone = C \sqrt{B\paren{f_\star}}$ and $\Bvtwo = C \rkhsbound \delta^{-1}$, with $C$ being a constant. Then for any joint distribution $\sfP$, with probability at least $1-\delta$ over $\sets{\bW_m}_{m\in[M]}$ sampled according to \eqref{eqn:w-distribution} and $\sets{(\sequenceX{(j)},y_j)}_{j\in[n]}\sim_{iid} \sfP$, the excess risk is bounded by
\begin{equation}\label{equ::sample complexity}
\begin{aligned}
    L_D(\hat{f}^{\bW}_{M}) -  L_D(f_\star) \le 
    \tO \paren{ \sqrt{\rkhsbound} \bigg[\sqrt{\frac{1}{{n}}} + \sqrt{\frac{d^2\delta^{-1}}{M}} \bigg] }. 
\end{aligned}
\end{equation}
\end{theorem}

The proof of Theorem \ref{thm::sample complexity} is contained in Appendix \ref{sec::proof of sample complexity}. The proof mostly uses the Rademacher complexity bound for the supremum of empirical process. The main non-trivial technical challenge lies in showing the concentration of $\sup_{f \in \cV_M} |\hat{L}_D(f) - L_D(f)|$, which cannot be simply controlled due to the unboundedness of the infinity norm of functions in the target function class $\cV_M$. 
We dealt with this subtlety by a carefully decomposition of $ \sup_{f \in \cV_M} |\hat{L}_D(f) - L_D(f)|$. The seemingly unnatural constraint set (\ref{equ::ERM constraint}) is used in bounding different terms in this decomposition.

\subsection{Examples and comparison}
\label{sec::examples and comparison}

We next give the sample complexity for learning several examples of target functions using the random-feature attention model. We will compare its sample complexity for learning these functions with that of the standard random-feature model \cite{rahimi2007random} (thereafter, we call it the random-feature MLP model, in short \RFMLP~model). In the \RFMLP~model, we view $\sequenceX{}$ as an input vector instead of a sequence of vectors denoted as $\operatorname{vec}(\sequenceX{}) = [\bx_0;\bx_1;\ldots;\bx_N;1] \in \R^{d (N+1)+1}$. The \RFMLP~is given by
\begin{equation}
 \tflp(\sequenceX{}; \bv) = \sum_{m=1}^M \sigma\big(\<\bw_m,\operatorname{vec}(\sequenceX{})\> \big) \cdot v_m, ~~~~~ \{ \bw_m \}_{m \in [M]} \simiid \normal(\bzero, \id/(N+2)). \label{eqn::MLP model}
\end{equation}
We choose the variance of random weights $\bw_m$ to be $1/(N+2)$ to ensure that $\langle \bw_m,\operatorname{vec}(\sequenceX{})\rangle \sim \normal(0, 1)$ has unit variance. The generalization and approximation properties of the random-feature MLP model have been well-studied in the literature, for example, \cite{arora2019fine, bach2017breaking,mei2022generalization2}. %

We instantiate Theorem~\ref{thm::sample complexity} on three concrete examples of target functions (calculations of the excess risks in Appendix \ref{sec:proof_examples_RFA}, where the result for \RFMLP{} are adapted\footnote{By deriving the corresponding results for Random Features instead of Neural Tangent Kernels.
} from~\cite{arora2019fine}). 
In all three cases, the target functions are permutation invariant with respect to $\sets{\bx_i}_{i\in\brac{N}}$, by which we naturally expect \RFA~to achieve better sample complexity than \RFMLP{} in accordance with this structure.

\begin{example}[Functions of $\bx_0$]
\label{exp:functions of x0}
We consider functions of $\bx_0$ (no dependence on $\bx_{1:N}$) of the form 
\[
f_\star(\sequenceX{})=\sum_{k=0}^\infty \<\bx_0^{\otimes k}, \bA_k\>,~~ \bA_k\in\R^{d^k},~~~~~ \text{with } \rkhsbound= \sum_{k=0}^\infty C_k \lfro{\bA_k}^2 \text{ by (\ref{eqn:bfstar})}. 
\]
By Theorem~\ref{thm::sample complexity}, setting $M = \Theta( d^2 n )$, the excess risk bound gives $\tO ( \sqrt{\sum_{k=0}^\infty k^{4.5}4^k \| \bA_k \|_{\Fr}^2 / n} )$. 
\end{example}

As a special case, consider $f_\star(\bx_{0:N})=(\bbeta^\top \bx_0)^p$, which corresponds to taking $\bA_k = \boldsymbol{\beta}^{\otimes p}$ for $k=p$ and $\bA_k = \boldsymbol{0}$ for $k \neq p$. The above excess risk of \RFA~model and the \RFMLP{} model scales as
\[
\RFA: \tO \Big(\Poly(p) \sqrt{ 4^p \| \bbeta \|_2^{2p} /n } \,\Big),~~~~~~~{\RFMLP}:  \tO \Big( \Poly(p) \sqrt{(N+2)^p \| \bbeta\|_2^{2p}/n} \Big).
\]

Compared to the \RFMLP~model, the \RFA~model significantly reduces the necessary sample size by a factor of $(N / 4)^p$. %

\begin{example}[Average of functions of $\bx_i$]
\label{exp:functions of xi}
We consider average of functions of $\bx_i$ of the form
\begin{align*}
f_\star(\sequenceX{})=\frac{1}{N}\sum_{i=1}^N \sum_{k=0}^\infty \langle \bx_i^{\otimes k}, \bA_k\rangle,~\bA_k\in\R^{d^k}, \quad \textrm{with}~
\rkhsbound= \sum_{k=0}^\infty C_k \lfro{\bA_k}^2~\textrm{by~\eqref{eqn:bfstar}}.
\end{align*}
Theorem~\ref{thm::sample complexity} then gives an $\tO ( \sqrt{\sum_{k=0}^\infty k^{4.5}4^k \| \bA_k \|_{\Fr}^2 / n} )$ excess risk, same as Example \ref{exp:functions of x0}.
\end{example}

As a specific example, consider $f_\star = \frac{1}{N} \sum_{i=1}^N \psi(\<\bbeta,\bx_i\>)$ with $\psi(z) = z\arctan(z/\eta)$ for some $\eta>2$, $\ltwos{\bbeta} = 1$. Using the power series expansion of $\psi$, the excess risk bound of \RFA{} model and the \RFMLP{} model scale as 
\begin{align*}
\textstyle
\RFA{}: \tO\paren{\sqrt{ \sum_{k=1}^\infty k^{4.5}(2 /\eta)^{2k} / n }} = \tO(\sqrt{1/n}), \quad 
\RFMLP{}: \tO\paren{\sqrt{ \sum_{k=1}^\infty k^{4.5}[(N+2) / (2\eta)]^{2k}  / n}}.
\end{align*}
The latter diverges whenever $\eta\le (N+2)/2$, in which case the bound is meaningless. %

\begin{example}[Correlation-weighted functions]
\label{exp:general correlation functions}
$f_\star$ is the following function:
\begin{align*}
    f_\star(\sequenceX{})=\frac{1}{N}\sum_{i=1}^N F(\<\bx_0, \bS\bx_i\>) G(\bx_i), ~~~~ F(t) = \sum_{k = 0}^\infty a_k \cdot t^k, ~~~~ G(\bx_i) = \sum_{k=0}^\infty \< \bx_i^{\otimes k},\bG_k\>,
\end{align*}
for $\bS \in \R^{d \times d}$, $\{ a_k\}_{k \ge 0} \subseteq \R$, $\bG_k \in \R^{d^k}$. This target function fully exploits the representation power of the attention layer. Eq. (\ref{eqn:bfstar}) gives $\rkhsbound = \cO( \sum_{k=0}^\infty C_{k} (\sum_{r+s=k} a_r^2 \lfros{\bS}^{2r}\lfros{\bG_{s}}^2 ))$. 
\end{example}

As a specific example, consider $f_{1, \star} = \frac{1}{N} \sum_{i=1}^N \<\bx_0,\bx_i\>^p$, corresponding to taking $\bS = \bI_d$, $F(t) = t^p$, and $G \equiv 1$. The excess risk bound of \RFA{} (by Theorem~\ref{thm::sample complexity}) and \RFMLP{} scale as 
\[
\textstyle
\RFA{}: \tO\paren{\Poly(p) \sqrt{ (4d)^p / n} },~~~~~~~ \RFMLP{}: \tO\paren{\Poly(p) \sqrt{ [(N+2)d]^p / n} }.
\]
As another example, consider $f_{2, \star} = \frac{1}{N} \sum_{i=1}^N \cos(\langle \bx_0,\bx_i\rangle) \langle \bx_i^{\otimes p},\bG\rangle$ with $\| \bG \|_{\Fr} = 1$. Then the excess risk bound of \RFA{} and \RFMLP{} scale as 
\[\textstyle
\RFA{}: \tO\paren{\Poly(p d) \sqrt{ e^{4\sqrt{d}} 4^p / n}}, ~~~~~ \RFMLP{}: \tO\paren{\Poly(p N d) \sqrt{ e^{2{(N+2)\sqrt{d}}} N^p /n}}. 
\]
{\RFA} reduces the required sample size by factors of $(N / 4)^p$ for $f_{1, \star}$ and $\exp({N\sqrt{d}})$ for $f_{2, \star}$.

\section{Expressivity of biased random-feature attention model}\label{section:bias attention}

\begin{figure}[t]
\hspace{0cm}
\includegraphics[width=1.0\linewidth,bb=0 0 7400 2880]{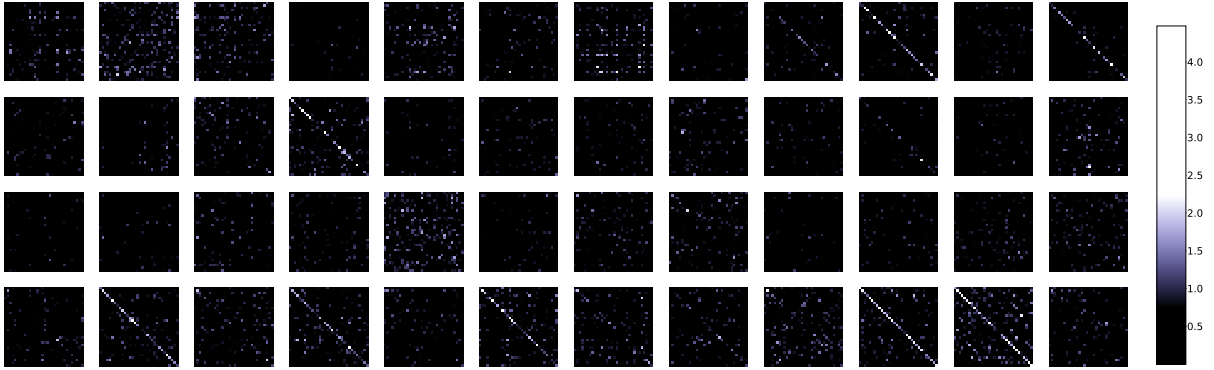}
\hspace{0cm}
\caption{Visualization of weight matrices of the 2nd, 5th, 8th, and 11th layers of the BERT-Base model. Each row contains weight matrices of a layer. All matrices are clipped to the top-left $32\times 32$ block. Lighter color indicates a larger absolute value.}
 \label{fig:bert visualization}
\end{figure}

We now move beyond the Gaussian weight assumption by exploring alternative possibilities for the weight distribution in the attention heads. We observe empirically that the weight matrices in transformer architectures learned in practice are often more similar to the identity matrix than a mean-zero matrix (Figure~\ref{fig:bert visualization}; see the details in Appendix \ref{app:bert}. This is also observed in a recent and concurrent work \cite{trockman2023mimetic}).

Towards understanding this effect, we consider an alternative attention model with \emph{biased} random weights, where the bias is a fixed matrix $\bW_0\in\R^{(d+1)\times(d+1)}$:
\begin{align}\label{eqn:RF_attention_bias}
 \tfwmb(\bx_{0:N}; \bV) = \sum_{m=1}^M \frac{1}{N}\sum_{i=1}^N \sigma\paren{\< \bW_0+\bW_m, \tbx_0\tbx_i^{\top}\>} \<\bv_m, \tbx_i\>. 
\end{align}
Here $\{ \bW_m \}_{m \in [M]}$ are again Gaussian random matrices sampled according to (\ref{eqn:w-distribution}). The biased random-feature attention model is similar to (\ref{eqn:RF_attention}) except that a bias weight $\bW_0$ is added. Motivated by our observation, we choose $\bW_0 = [\bId, \bzero_{d \times 1};
    \bzero_{1 \times d} , 0] \in \R^{(d + 1) \times (d+1)}$, so that the diagonal elements of $\bW_0 + \bW_m$ will be on average larger than the off-diagonal elements.

\subsection{Expressivity of biased random-feature attention} 

Given the formulation of biased random-feature attention models (thereafter, we call it the biased random-feature attention model, in short \BRFA~model), a natural conjecture is that this model can better fit functions that are the average of function of $\< \bx_0, \bx_i\>$. We here show that this is indeed the case. In particular, we consider a broad class of target functions $ \fhard:\cX\to\R$ that take forms %
\begin{align}
    &~\textstyle  \fhard(\bx_{0 : N}) = \frac{1}{N} \sum_{i=1}^N F(\langle \bx_0,\bx_i \rangle) G(\bx_0,\bx_i) \nonumber \\
    &~\textstyle F(t) = \sum_{k=0}^\infty a_k t^k, \quad G(\bx_0, \bx_i) = \<\tbx_i^{\otimes 3} \otimes \tbx_0^{\otimes 2}, \bA_\star\>. \label{eqn:bias rf hard} %
\end{align}
Here 
the scalers $\{ a_k \}_{k \ge 0} \subseteq \R$ and the tensor $\bA_\star \in \R^{d^5}$ parameterizes $\fhard$. As we will explain in Section \ref{sec:overview_technique}, confining $G$ to be a degree-$(3, 2)$ polynomial in $(\bx_i, \bx_0)$ is essential to our theoretical results. Our next theorem provides the excess risk of learning target function $\fhard$ using the {\BRFA} model (\ref{eqn:RF_attention_bias}). %

\begin{theorem}
\label{thm::bias sample complexity}
Given the same setting and assumptions as in Theorem \ref{thm::sample complexity}, 
when the population risk minimizer gives $f_\star = \fhard$, 
with probability at least $1-\delta$, we have 
\begin{align}
\populoss{\hat{f}^{\bW,\bW_0}_{M}} -  L_D(\fhard) &=  \tO \Bigg( \inf_{L} \bigg[  \sqrt{B(\fhard, L)} \Big( 
 \sqrt{\frac{1}{n}} + \sqrt{\frac{d^2 \delta^{-1}}{M}}\Big) + \eps_L \linf{\fhard}\bigg] \Bigg), \label{eqn:bias_RF sample complexity line 2}
\end{align}
where $\epsilon_L = 1/[2^{L+1}(L+1)!]$ and
\begin{align}
\textstyle
    B(\fhard, L) = \lfro{\bA_\star}^2 \cdot 
( \sum_{k=0}^\infty | a_k | \cdot C_k )^2, ~~~~\text{\rm with } C_k = (2L+k)^{(k+3)/2} 8^{L+ k/2}. 
\label{eqn:bfstar for biased}
\end{align}

\end{theorem}
The proof of Theorem \ref{thm::bias sample complexity} is contained in Appendix \ref{sec::bias sample complexity proof}. We provide the intuitions of the result and an overview of the proof technique in Section \ref{sec:overview_technique}. %

\subsection{Examples and comparison} \label{sec::examples and comparison2}
    
    Compared to the target functions \eqref{eqn:target-expansion} discussed in Section \ref{sec:expressivity of rf}, functions in \eqref{eqn:bias rf hard} may not express the average of arbitrary functions of $\bx_0$ and $\bx_i$, but are well-suited to express functions of correlations. Consequently, we anticipate that the {\BRFA} model will outperform the {\RFA} model in learning functions of correlations. We will now present three concrete examples of target functions \eqref{eqn:bias rf hard}, and compare the excess risk of the {\BRFA} model to that of the {\RFA} model. The proof of excess risk is contained in Appendix \ref{sec:proof_examples_bias}.

\begin{example}[Low degree polynomials]
\label{exmp::low degree polynomials}
Consider average of polynomials of $\bx_i$ and $\bx_0$,
    \[ \fhard = \frac{1}{N}\sum_{i=1}^N \langle  \bx_i^{\otimes 3}\otimes  \bx_0^{\otimes 2} , \bA\rangle, ~~~ \text{\rm~with } B(\fhard,L) = \| \bA \|_{\Fr}^2 L^{3} 8^{2L} ~~\text{\rm~by \eqref{eqn:bfstar for biased}}.\]
    For any $\eta>0$, if we take $n\ge \exp(\exp(\Theta(1/\eta)))$, $L=\Theta((1 + \log\log n)^{-1}\log n)$, and $M = \Theta( d^2 n)$, the excess risk will scale as $\tO(\sqrt{\| \bA \|_{\Fr}^2/n^{1-\eta}})$. 
\end{example}

Compared with the excess risk of the {\RFA} model as detailed in Example \ref{exp:functions of xi}, the excess risk bound of the {\BRFA} model loses a factor of $n^{-\eta/2}$.

\begin{example}[Functions of correlations]
\label{exp:correlation functions biased}
    Consider a special case of functions of correlations,
    \[ \fhard = \frac{1}{N} \sum_{i=1}^N \<\bx_0,\bx_i\>^p\<\bbeta,\bx_i\>,~~~~ \bbeta \in \S^{d-1},~~ \text{\rm with } B(\fhard,L) = (2L+p)^{p+3} 8^{2L+p} ~~\text{\rm~by \eqref{eqn:bfstar for biased}}.\]
    For any $\eta>0$, choosing the same parameters $(n, L, M)$ as Example~\ref{exmp::low degree polynomials}, the excess risk bound scales as $\tO(\sqrt{{ (\log n + p)^{(p+3)}8^p}/{n^{1-\eta}}})$.
\end{example}

Consider the required sample size $n_\star$ to reach an accuracy of $0.01$. The \BRFA~model requires $n_\star = \tO((8p+48)^{p+3})$, whereas the \RFA~model requires $n_\star = \tO((4d)^p)$. Thus, in comparison to the \RFA~model, the \BRFA~model can reduce the required sample size by a factor of $\tO([d/(2p+12)]^p)$. 

\begin{example}[Correlation-weighted functions]
\label{exp:general correlation functions biased}
Consider the function
    \[ \fhard = \frac{1}{N} \sum_{i=1}^N \cos(\<\bx_0,\bx_i\>) \<\bx_i^{\otimes 3},\bG\>,~~~~ \text{ with } \lfro{\bG}^2 \le 1 \text{ and } B(\fhard, L) = \Theta((8e)^{2L}), \]
    where $B(\fhard,L)$ is bounded through the Taylor expansion of $\cos(t)$ and \eqref{eqn:bfstar for biased}.
    For any $\eta>0$, choosing the same parameters as Example \ref{exmp::low degree polynomials}, the excess risk bound scales as $\tO(\sqrt{1/{n^{1-\eta}}})$.
\end{example}

Consider the required sample size $n_\star$ to reach an accuracy of $0.01$. The \BRFA~model requires $n_\star = \tO(1)$, whereas the \RFA~model requires $n_\star = \tO(\Poly(d)\exp(\sqrt{d}))$. Thus, in comparison to the \RFA~model, the \BRFA~model can reduce the required sample size by a factor of $\tO(\Poly(d)\exp(\sqrt{d}))$.

\subsection{Overview of techniques}\label{sec:overview_technique} 

Here we provide the intuition and an overview of the technique of Theorem \ref{thm::bias sample complexity}, with the proof details in Appendix \ref{sec::bias sample complexity proof}. To show the sample complexity of learning with the {\BRFA} model, the first step is to derive the kernel $K_{\BRFA}(\bx_{0:N}, \bx_{0, N}')$ associated with the infinite-width {\BRFA} model. This kernel has a natural feature map, given by $\{ \Psi_{k}: \cX \to \R^{d^{2k+1}} \}_{k \ge 0}$, where
\[
\textstyle \Psi_k(\bx_{0:N}) = \sum_{i = 1}^N \phi(\<\bx_0, \bx_i\>) \cdot \He_{k-2}(\<\bx_0, \bx_i\>) \cdot \tbx_i^{\otimes k+1} \otimes \tbx_0^{\otimes k}, ~~~ \forall k \ge 2. 
\]
Here $\phi(t) = (2 \pi)^{-1/2} e^{-t^2/2}$ is the Gaussian density function, and $\He_k(z)$ denotes the $k$-th probabilist's Hermite polynomial, with detailed expression and properties given in Appendix \ref{sec:Hermite}. This feature map implies the learnability of the following target function class by the {\BRFA} model, 
\begin{align}
 \textstyle    \feasy(\bx_{0:N}) = \frac{1}{N}\sum_{i=1}^N \phi(\<\bx_0, \bx_i\>) \sum_{k=2}^\infty \He_{k-2}(\<\bx_0, \bx_i\>) \<\tbx_i^{\otimes k+1} \otimes \tbx_0^{\otimes k}, \bA_k\>, \label{eqn:bias rf easy}
\end{align}
whose RKHS norm associated with kernel $K_{\BRFA}$ is bounded by $B(\feasy) = \sum_{k=2}^\infty (k-2)! k^2 4^k \|\bA_k\|_{\sf Fr}^2$.

Notice that $\feasy$ bears similarities to, but also distinct differences from, $\fhard$ as presented in (\ref{eqn:bias rf hard}). The key difference lies in the $\phi(\langle\bx_0, \bx_i\rangle)$ factor in $\feasy$, which is hard to interpret and analyze. To obtain the excess risk bound for learning $\fhard$, we can use $\feasy$ to approximate $\fhard$ in the $L^\infty$ norm. The excess risk for learning $\fhard$ can be bounded by the summation of the excess risk for learning $\feasy$ and the approximation error. Acquiring this approximation error bound necessitates a truncation argument of the Taylor expansion of $1/\phi(\cdot)$.

\section{Numerical experiments}\label{sec:experiments}

We test our theory by experimentally approximating two types of target functions using the three models under investigation \RFA~\eqref{eqn:RF_attention}, \BRFA~\eqref{eqn:RF_attention_bias}, and \RFMLP~\eqref{eqn::MLP model}. We choose the target functions to be of form 
\begin{align}
 \textstyle   f_{1,p}(\bx_{0:N})  =&~ \textstyle \frac{1}{N}\sum_{i=1}^N \<\bbeta,\bx_i\>^p, ~~~~~~~~~~~~~~~~~~~ p\in \N, ~~~~~~ \bbeta\in \S^{d-1}, \label{eqn:experiment rf_easy}\\
 \textstyle    f_{2,q}(\bx_{0:N})=&~ \textstyle  \frac{1}{N}\sum_{i=1}^N \<\bx_0,\bx_i\>^q\<\bbeta,\bx_i\>,~~~~~~ q\in \N , ~~~~~~ \bbeta\in \S^{d-1} \label{eqn:experiment rf_hard}.
\end{align}

The first target function \eqref{eqn:experiment rf_easy} is a specific instance of Example \ref{exp:functions of xi}, whereas the second target function \eqref{eqn:experiment rf_hard} has been considered in both Example \ref{exp:general correlation functions} and \ref{exp:correlation functions biased}.

In our experimental setup, we set the input dimension as $d = 16$ and the number of tokens as $N = 16$. We fix the width of {\RFA} and {\BRFA} to be $M_{\RFA} = M_{\BRFA} = M = 1000$, whereas the width of {\RFMLP} is set as $M_{\RFMLP} = M(d + 1) = 17000$. This configuration ensures an equal number of parameters across all three models. To further accentuate the test risk difference between the {\BRFA} and {\RFA}, in {\BRFA} we use a bias matrix of $\bW_0 = 4 [ \bId, \bzero_{d \times 1}; \bzero_{1 \times d}, 0] \in \R^{(d+1) \times (d+1)}$, which is four times the matrix investigated in our theory. The input distribution is selected as $\{ \bx_i \}_{0 \le i \le N} \simiid \Unif(\S^{d-1})$, and we take $y = f_\star(\sequenceX{})$ without any noise. We consider three representative target functions: $f_{1,p}$ for $p = 2, 4$, and $f_{2, p}$ for $p = 3$, as per (\ref{eqn:experiment rf_easy}) and (\ref{eqn:experiment rf_hard}). We examine a list of sample sizes $n$ from $2^4$ to $2^{12}$. Prior to training with RF models, we standardize the $y_i$'s to have zero mean and unit standard deviation, ensuring that the trivial risk equals $1$. We train the RF models using square loss with ridge regularization, selecting the ridge parameter to minimize the test error. The experimental results are displayed in Figure \ref{fig:Allfigs}.

\begin{figure}
\hspace{-2.4cm}
\includegraphics[width=1.26\linewidth,bb= 0 0 1656 360]{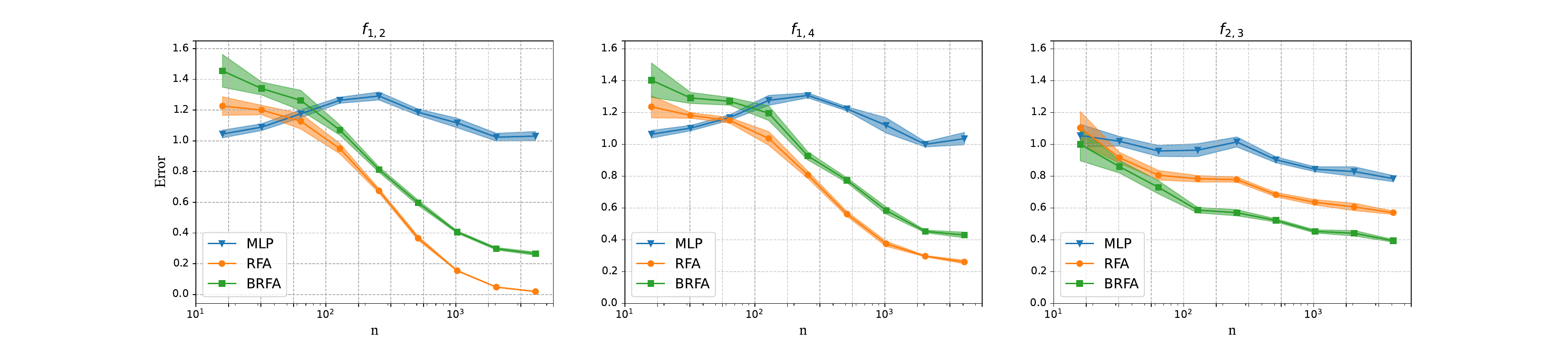}
\vspace{-0.3cm}
\caption{Test error of three \RF~models for learning $f_{1, 2}$ (left), $f_{1, 4}$ (mid), and $f_{2, 3}$ (right), as per (\ref{eqn:experiment rf_easy}) and (\ref{eqn:experiment rf_hard}). We set $d, N = 16$, $M_{\RFA} = M_{\BRFA} = 1000$, and $M_{\RFMLP} = 17000$. We train the \RF~models using square loss with ridge regularization, with the ridge parameter selected to minimize test error. The test error is calculated using $n_{\rm test} = 1000$ fresh samples. The figure reports the mean and normalized standard error of the test error, based on $5$ independent experimental instances. }

\label{fig:Allfigs}
\end{figure}

The left and middle panels of Figure \ref{fig:Allfigs} demonstrate a noticeable separation between {\RFMLP} and the other two random-feature attention models for learning these target functions. {\RFMLP} can hardly approximate the target function, whereas {\RFA} and {\BRFA} exhibit significantly better performance. This observation is consistent with our sample complexity analysis detailed in Example \ref{exp:functions of xi}, where the sample complexity bound of {\RFMLP} for learning average of functions of $\bx_i$ is found to be $\cO((N/4)^p)$ times greater than that of {\RFA}.

The performance comparison between {\RFA} and {\BRFA} depends on the target functions. {\RFA} outperforms {\BRFA} in learning $f_{1, 2}$ and $f_{1, 4}$, whereas {\BRFA} outperforms {\RFA} in learning $f_{2, 3}$. The latter phenomenon is as we expected: as demonstrated in Example \ref{exp:general correlation functions} and \ref{exp:correlation functions biased}, {\BRFA} is more powerful than {\RFA} in approximating the correlation-weighted functions.

We have conducted further experiments with various other target functions, detailed in Appendix \ref{sec::further experiments}.

\section{Conclusion}\label{sec:discussions}

In this work, we introduced and examined the expressivity of two random-feature attention models, namely \RFA~\eqref{eqn:RF_attention} and \BRFA~\eqref{eqn:RF_attention_bias}. For general classes of functions that are invariant to the permutation of key tokens $\bx_{1:N}$, the excess risk of {\RFA} \eqref{eqn:RF_attention} can avoid the dependence on sequence length, in contrast to the standard random-feature model {\RFMLP} \eqref{eqn::MLP model}. Moreover, for specific functions that adopt the form of correlation-weighted polynomials \eqref{exp:general correlation functions biased}, the excess risk of {\BRFA} can avoid the polynomial dependence on the dimension. %
These insights enhance our understanding of the attention mechanism within a simplified context. Finally, our work left open many interesting questions for future work, such as the expressivity of softmax attention, the influence of positional encoding in expressivity, and the expressivity of multi-layer transformers. %

\section*{Acknowledgment}

S. Mei is supported in part by NSF DMS-2210827 and NSF CCF-2315725. 

\bibliographystyle{alpha}
\bibliography{bib}

\makeatletter
\def\renewtheorem#1{%
  \expandafter\let\csname#1\endcsname\relax
  \expandafter\let\csname c@#1\endcsname\relax
  \gdef\renewtheorem@envname{#1}
  \renewtheorem@secpar
}
\def\renewtheorem@secpar{\@ifnextchar[{\renewtheorem@numberedlike}{\renewtheorem@nonumberedlike}}
\def\renewtheorem@numberedlike[#1]#2{\newtheorem{\renewtheorem@envname}[#1]{#2}}
\def\renewtheorem@nonumberedlike#1{  
\def\renewtheorem@caption{#1}
\edef\renewtheorem@nowithin{\noexpand\newtheorem{\renewtheorem@envname}{\renewtheorem@caption}}
\renewtheorem@thirdpar
}
\def\renewtheorem@thirdpar{\@ifnextchar[{\renewtheorem@within}{\renewtheorem@nowithin}}
\def\renewtheorem@within[#1]{\renewtheorem@nowithin[#1]}
\makeatother

\renewtheorem{theorem}{Theorem}[section]
\renewtheorem{lemma}{Lemma}[section]
\renewtheorem{remark}{Remark}
\renewtheorem{corollary}{Corollary}[section]
\renewtheorem{observation}{Observation}[section]
\renewtheorem{proposition}{Proposition}[section]
\renewtheorem{definition}{Definition}[section]
\renewtheorem{claim}{Claim}[section]
\renewtheorem{fact}{Fact}[section]
\renewtheorem{assumption}{Assumption}[section]
\renewcommand{\theassumption}{\Alph{assumption}}
\renewtheorem{conjecture}{Conjecture}[section]

\appendix

\section{Technical tools}
\subsection{Basics on Hermite Polynomials}\label{sec:Hermite}
In this section, we briefly review Hermite polynomials and their properties. Let $\He_n$ be the probabilists' Hermite polynomial of degree $n$:
\begin{equation*}\He_n(x) \defeq (-1)^n e^\frac{x^2}{2} \frac{d^n}{dx^n}e^{-\frac{x^2}{2}}.\end{equation*}
Here are some basic properties of Hermite polynomials $\{ \He_n\}_{n \ge 0}$:
\begin{itemize}[leftmargin=2em]
    \item Let $(x,y) \sim \normal(\bzero, [1, \rho; \rho, 1])$. Then $\E \brac{\He_m(x)\He_n(y)} = n!\rho^n 1_{m = n}$ for $m,n \ge 0$. 
    \item $\He_n(-x)=(-1)^n \He_n(x);$
    \item $\He_{n+1}(x)=x\He_n(x)-\He_n'(x).$
    \item $\He_{n}'(x)=n\He_{n-1}(x).$
\end{itemize}

It can be shown that Hermite polynomials are a complete orthogonal basis of the following Hilbert space with $\phi(x) = (2 \pi)^{-1/2} e^{- x^2/2}$:   
\begin{equation*}
L^2(\R, \phi(x)) \defeq \left\{f : \R \rightarrow \R:~~~ \int_{-\infty}^{\infty} f(x)^2 \phi(x) \de x < \infty \right\}.
\end{equation*}
For two functions $f,g: \R \to \R$, we define their inner product  $\<f,g\>$ as:
\begin{equation*}
\<f,g\> \defeq \E_{x\sim \normal(0,1)} \brac{f(x)g(x)}=\int_{-\infty}^{\infty} f(x)g(x)\phi(x) \de x . 
\end{equation*}
For any function $f \in L^2(\R, \phi(x))$, we can derive its Hermite expansion:
\begin{equation*}f(x)=\sum_{n = 0}^{\infty} \frac{a_n}{n!} \He_n(x), \end{equation*}
where $a_n=\<f,\He_n\> =\int_{-\infty}^{\infty} f(x)\He_n(x)\phi(x) \de x$.
Then for another function $g(x)=\sum_{n = 0}^{\infty} \frac{b_n}{n!} \He_n(x)$, the inner product of $f$ and $g$ gives:
\begin{equation*}\<f,g\>=\sum_{n = 0}^{\infty} \frac{a_n b_n}{n!}.\end{equation*}
Here are some formulae for Hermite expansion of certain functions. 
\begin{lemma}[See e.g., \cite{ge2017learning}] \label{lemma:Hermite inner product}
    Let $f,g \in L^2(\R, \phi(x) )$.  Then, for any unit vectors $u,v\in \R^d$, we have that 
    \begin{equation*}\E_{x\sim \mathcal{N} (0, \mathbf{I}_d )} \brac{f(u^{\top}x)g(v^{\top}x)}  = \sum_{n = 0}^{\infty} \frac{a_n b_n}{n!} \<u,v\> ^n, \end{equation*}
    where $f(x)=\sum_{n = 0}^{\infty} \frac{a_n}{n!} \He_n(x)$ and $g(x)=\sum_{n = 0}^{\infty} \frac{b_n}{n!} \He_n(x).$
\end{lemma}

\begin{lemma}[Inverse explicit expression \cite{patarroyo2019digression}] Hermite expansion of $x^n$ gives:
\label{clm:Inverse explicit expression}
    \begin{equation*}x^n=n!\sum_{m=0}^{\lfloor \frac{n}{2} \rfloor}\frac{1}{2^mm!(n-2m)!}\He_{n-2m}(x).\end{equation*}
\end{lemma}

\begin{lemma}
\label{clm::relu with hermite}
    Let $\barsig_c(x) \defeq \barsig(x+c)$ be the shifted ReLU function. Then the Hermite expansion of $\barsig_c$ gives:
    \begin{equation*}\barsig_c(x) =  c\Phi(c) + \phi(c)+ \Phi(c)\He_1(x)+ \sum_{n = 2}^{\infty} \frac{(-1)^n}{n!}  \phi(c) \He_{n-2}(c)\He_n(x), \end{equation*}
    where $\phi(x), \Phi(x)$ are the PDF and CDF of standard Gaussian.
\end{lemma}
\begin{proof}[Proof of Lemma \ref{clm::relu with hermite}]
Denote $a_{n,i}\defeq \int_{-c}^{\infty}{(x+c)^i\He_n(x)\phi(x)dx}$. It is easy to obtain that
\[
    a_{0,0}=\Phi(c),~~~~~~ a_{0,1}=c\Phi(c)+\phi(c),~~~~~~ a_{1,0}=\phi(c),~~~\text{and}~~~ a_{1,1}=\Phi(c).
\]
Then using the two different formulae for $a_{n, 0}$, 
\begin{align*}
    a_{n,0}&=\int_{-c}^{\infty}{\He_n(x)\phi(x)\de x}=-\frac{\He_{n+1}(-c)}{n+1}\phi(-c)+\int_{-c}^{\infty}{\frac{\He_{n+1}(x)}{n+1}x\phi(x)\de x}\\
    &=-\frac{\He_{n+1}(-c)}{n+1}\phi(-c)+\frac{1}{n+1}a_{n+1,1}-\frac{c}{n+1}a_{n+1,0},~~~n\ge 0,~~~\text{and}\\
    a_{n,0}&=\int_{-c}^{\infty}{(x\He_{n-1}(x)-(n-1)\He_{n-2}(x))\phi(x)\de x}\\
    &=a_{n-1,1}-ca_{n-1,0}-(n-1)a_{n-2,0},~~~n\ge 2,
\end{align*}
we obtain that
\begin{align*}
    a_{n,1}&=(-1)^n\He_n(c)\phi(c)+ca_{n-1,1}+(n-c^2)a_{n-1,0}-(n-1)ca_{n-2,0},~~~\text{and}\\
    a_{n,0}&=a_{n-1,1}-ca_{n-1,0}-(n-1)a_{n-2,0},~~~n\ge 2.
\end{align*}
Then it is easy to prove by induction that
\begin{align*}
    a_{n,1}&=(-1)^n\He_{n-2}(c)\phi(c),~~~ n\ge2,\\
    a_{n,0}&=(-1)^{n-1}\He_{n-1}(c)\phi(c),~~~n\ge 1.
\end{align*}
This completes the proof. 
\end{proof}

\subsection{Basics on Random Features}
In this section, we give some basic properties of the random feature model considered in our work, which can be seen as an extension of the standard random feature model (e.g. of~\cite{rahimi2007random,rahimi2008uniform}) to the \emph{vector-valued case}.

Given a functional $\bsigma(x ; w): \mathcal{X} \times \mathcal{W} \rightarrow \mathbb{R}^d$. Denote $\mu$ as a probability measure on $\mathcal{W}$. We define the (infinite-width) random feature model as:
\begin{equation}
    \mc{F}=\sets{f:f(x) = \<\bv,\bsigma(x;\cdot)\>_{\mc{H}_\mc{W}} ,~~~\bv\in \mc{H}_\mc{W}},
\end{equation}
where $\mc{H}_\mc{W}=\set{\bv(w):\int_{\mathcal{W}}\bv(w)^{\top}\bv(w)\mu(\de w) <\infty\label{eqn:basic of rf 1}}$ is a Hilbert space with norm $\norm{\bv}_{\mc{H}_\mc{W}}^2=\int_{\mathcal{W}}\bv(w)^{\top}\bv(w)\mu(\de w) $ and inner product $\<\bv,\bu\>_{\mc{H}_\mc{W}}= \int_{\mathcal{W}}\bv(w)^{\top}\bu(w)\mu(\de w)$.
Besides, we endow $\mc{F}$ with a norm $\norm{\cdot}_{\mc{F}}$ and the corresponding inner product $\<\cdot,\cdot\>_{\mc{F}}$ defined as:
\begin{equation*}
    \norm{f}_{\mc{F}}=\inf_{f=\<\bv,\bsigma(\cdot)\>_{\mc{H}_\mc{W}}}{\norm{\bv}_{\mc{H}_\mc{W}}}, ~~~\<f,g\>_{\mc{F}}=\frac{\norm{f+g}_{\mc{F}}^2-\norm{f-g}_{\mc{F}}^2}{4}.
\end{equation*}
We further define the corresponding reproducing kernel $K:\mathcal{X}\times\mathcal{X} \rightarrow \R$, s.t.
\begin{equation*}
K(x, y)=\int_{\mathcal{W}} \bsigma(x ; w)^{\top} \bsigma(y ; w) \mu(\de w),
\end{equation*}
which is positive definite. Define the RKHS induced by this kernel as $\mathcal{H}_K$ with corresponding norm $\norm{\cdot}_{\mc{H}_K}$ and the inner product $\<\cdot,\cdot\>_{\mc{H}_K}$. Then we have the following proposition according to \cite{minh2006mercer}:
\begin{proposition}
    Given the above definition of $\mc{F}$ and $\mc{H}_{K}$, we have that $\paren{\mc{F},\norm{\cdot}_{\mc{F}}}=\paren{\mc{H}_{K},\norm{\cdot}_{\mc{H}_K}}$.
    
\end{proposition}
More generally \cite{bai2019beyond}, for any feature map $\phi:\mc{X}\rightarrow \mc{H}$ (where $\mc{H}$ is a Hilbert space) that induces the kernel $K$, i.e., $K(x,y)=\<\phi(x),\phi(y)\>_\mc{H}$, we have for any function $f$ that 
\begin{equation}\label{eqn:feature map equi}
\norm{f}_{\mc{H}_K}=\inf_{f=\<\bu,\phi(\cdot)\>_{\mc{H}}}{\norm{\bu}_{\mc{H}}},
\end{equation}
which shows the equivalence among different feature maps that generate the same kernel.
\subsection{Concentration inequalities}
\begin{definition}[Sub-Gaussian and Sub-Exponential random variables \cite{vershynin2018high}] \label{def:subgaussian norm}
For a random variable  $X$, its sub-gaussian norm, denoted $\|X\|_{\psi_2}$, is defined as 
\begin{equation*}
\|X\|_{\psi_2}=\inf \left\{t>0: \mathbb{E} \exp \left(X^2 / t^2\right) \leq 2\right\}.
\end{equation*}
If $\sigma \equiv \|X\|_{\psi_2} < \infty$, we say that $X$ is $\sigma$-sub-Gaussian. 

For a random variable  $X$, its sub-exponential norm, denoted $\|X\|_{\psi_1}$, is defined as 
\begin{equation*}
\|X\|_{\psi_1}=\inf \left\{t>0: \mathbb{E} \exp \left(|X|/ t\right) \leq 2\right\}.
\end{equation*}
If $\sigma \equiv \|X\|_{\psi_1} < \infty$, we say that $X$ is $\sigma$-sub-exponential. 
\end{definition}

\begin{theorem}[Gaussian concentration inequality (e.g., \cite{wainwright2019high})]\label{thm:gaussian concentration}
Let $\left(X_1, \ldots, X_n\right)$ be a vector of i.i.d. standard Gaussian variables, and let $f: \mathbb{R}^n \rightarrow \mathbb{R}$ be L-Lipschitz with respect to the Euclidean norm. Then the random variable $f(X)-\mathbb{E}[f(X)]$ is $L$-sub-Gaussian, and hence
\begin{equation*}
\mathbb{P}\big( |f(X)-\mathbb{E}[f(X)]| \geq t \big) \leq 2 e^{-\frac{t^2}{2 L^2}} \quad \text { for all } t \geq 0. 
\end{equation*}
\end{theorem}

\begin{theorem}[Bounded difference inequality (e.g., \cite{wainwright2019high})] \label{thm:bounded difference}
Consider a function $f(X): \R^n\to\R$.
Assume that for any $X=(X_1,\ldots,X_n)$ and $X^{i,\prime} = (X_1,\ldots,X_i^\prime,\ldots,X_n)$, we have difference bound $|f(X) - f(X^{i,\prime})| \le L_i$.
We further assume that the random vector $X=\left(X_1, X_2, \ldots, X_n\right)$ has independent components. Then 
\begin{equation*}
\mathbb{P}\big( |f(X)-\mathbb{E}[f(X)]| \geq t \big) \leq 2 e^{-\frac{2 t^2}{\sum_{k=1}^n L_k^2}} \quad \text { for all } t \geq 0.
\end{equation*}
\end{theorem}

\begin{theorem}[Matrix Bernstein Inequality (e.g., \cite{tropp2015introduction})] \label{thm:matrix bernstein}
Consider a  sequence $\{\boldsymbol{S}_k \}_{k \in [n]}$ of independent random matrices with common dimension $d_1 \times d_2$. Assume that 
\begin{equation*}
\mathbb{E} \boldsymbol{S}_k=\mathbf{0} ~~~~~~\text { and }\left\|\boldsymbol{S}_k\right\|_{\op} \leq L\text { almost surely, ~~~~for each index } k.
\end{equation*}
Introduce the random matrix
\begin{equation*}
\boldsymbol{Z}=\sum_{k = 1}^n \boldsymbol{S}_k.
\end{equation*}
Let $v(\boldsymbol{Z})$ be the matrix variance statistic of the sum
\begin{equation*}
\begin{aligned}
v(\boldsymbol{Z}) & =\max \left\{\left\|\mathbb{E}\left(\boldsymbol{Z} \boldsymbol{Z}^\top\right)\right\|_{\op},\left\|\mathbb{E}\left(\boldsymbol{Z}^\top \boldsymbol{Z}\right)\right\|_{\op}\right\} \\
& =\max \left\{\left\|\sum_{k=1}^n \mathbb{E}\left(\boldsymbol{S}_k \boldsymbol{S}_k^\top\right)\right\|_{\op},\left\|\sum_{k=1}^n \mathbb{E}\left(\boldsymbol{S}_k^\top \boldsymbol{S}_k\right)\right\|_{\op}\right\} .
\end{aligned}
\end{equation*}
Then we have
\begin{equation*}
\mathbb{E}\|\boldsymbol{Z} \|_{\op} \leq \sqrt{2 v(\boldsymbol{Z}) \log \left(d_1+d_2\right)}+\frac{1}{3} L \log \left(d_1+d_2\right), 
\end{equation*}
and for all $t \geq 0$,
\begin{equation*}
\mathbb{P}\big( \|\boldsymbol{Z}\|_{\op} \geq t \big) \leq\left(d_1+d_2\right) \exp \left(\frac{-t^2 / 2}{v(\boldsymbol{Z})+L t / 3}\right).
\end{equation*}
\end{theorem}

\begin{theorem}[Ledoux-Talagrand contraction inequality (e.g., \cite{wainwright2019high})]\label{thm:rademacher contraction}
Let $\{ \xi_i \}_{i \in [n]} \simiid \Unif(\{ \pm 1\})$ be independent Rademacher random variables. For any set $\mathbb{T} \subset \R^n$ and any family of $L$-Lipschitz functions $\sets{\phi_j}_{j\in[n]}$ with $\phi_i(0) = 0$, we have 
\begin{equation}
    \E\Big[ \sup_{\theta\in\mathbb{T}} \Big| \sum_{i=1}^n \xi_i \phi_i(\theta_i) \Big| \Big]
    \le 
    2L \cdot \E\Big[ \sup_{\theta\in\mathbb{T}}\Big| \sum_{i=1}^n \xi_i \theta_i \Big| \Big].
\end{equation}
    
\end{theorem}

\section{Proofs for Section \ref{section:RKHS analysis}}
Throughout this section, we use the notation $\lesssim$ to hide a universal constant $C$. Also, we use $\bar\sigma(\bW, \bX)$ to denote a function of $\bW \in \R^{(d + 1) \times (d+1)}$ and $\bX \in \{ \overline\bX \in \R^{(d+1)\times(d+1)}: \| \overline\bX \|_{\Fr}^2 \le 4\} \equiv \cD$ that satisfies the following properties:
\begin{enumerate}
    \item For any $\bX \in \cD$, we have that $\bar\sigma(\bW, \bX)$ is $L_1$-Lipschitz with respect to $\bW$.
    \item For any $\bX \in \cD$, the expectation $\E_{\bW}[\bar\sigma(\bW, \bX)] \le L_2$.
    \item For any $\bX \in \cD$, we have $\bar\sigma(\boldsymbol{0},\bX) \le L_3$. 
\end{enumerate}
Here $L_1$, $L_2$, and $L_3$ are universal constants. Let $\{ \bW_m\}_{m \in [M]}$ be sampled from Eq. \eqref{eqn:w-distribution}. Let $\bv: \R^{(d+1 ) \times (d+1)} \to \R$ with $\E_{\bW}[ \ltwos{\bv\paren{\bW}}^2] \le R^2$. Consider a random feature model associated with $\bar\sigma$ (hereafter, we will refer to it as \RF~model)
\begin{equation}
    \tfwm(\bx_{0:N}; \bV) =  \sum_{m=1}^M \frac{1}{N}\sum_{j=1}^n \bar\sigma\paren{\bW_m, \tbX} \<\bv_m, \tbx_i\>, \label{eqn:finite RF}
\end{equation}
as well as the infinite-width version of the random feature model,
\begin{align}
    f_{\bv}\paren{\sequenceX{}} =  \frac{1}{N} \sum_{i=1}^N \E_{\bW} \bigbrac{\bar\sigma(\bW,\tbx_{0} \tbx_{i}^\top ) \<\bv\paren{\bW},\tbx_i\>}. \label{eqn:infinite RF}
\end{align}
Note that both model {\RFA} and {\BRFA} correspond to special choices of $
\bar\sigma$. Thus, all lemmas and theorems in this section are applicable to both model \RFA~and model \BRFA.

\subsection{Proof of Theorem \ref{thm:finite-width}}\label{sec:proof-finite-width}
To prove Theorem \ref{thm:finite-width}, we first state two auxilliary lemmas, Lemma \ref{lem::rf and finite neuron} and \ref{lem:RKHS_norm_target_function_Theorem1}.
\begin{lemma}[From infinite-width {\RF} model to finite-width {\RF} model]
\label{lem::rf and finite neuron}
Consider $f_\bv$ that takes the form as Eq. \eqref{eqn:infinite RF}, with $\E_{\bW}[ \ltwos{\bv\paren{\bW}}^2] \le R^2$.
 Let $ \sequenceX{} \sim \Px$. Define the $\norms{\cdot}^2_{L^2(\Px)}$ norm by
\[\norms{g}_{L^2\paren{\Px}}^2 = \int g\paren{\sequenceX{} }^2 \Px \paren{\de 
 \sequenceX{}}. \]
Then with probability at least $1-\delta$, there exists a sequence of vectors $\set{\bv_m}_{m=1}^M \subseteq \R^{d+1}$ and constant universal $C < \infty$ that only depends on $L_1$, $L_2$, and $L_3$ s.t.
\begin{align}
    \Big\| f_{\bv} - \frac{1}{N} \sum_{i=1}^N \sum_{m=1}^M \bar\sigma( \bW_m,\tbx_{0} \tbx_{i}^\top ) \<\bv_m,\tbx_i\> 
 \Big\|_{L^2\paren{\Px}}^2 \le   \frac{C (d^2 + \log M)R^2\delta^{-1}}{M} \label{equ::finite neuron}
\end{align}
\begin{equation}
 \sum_{m=1}^M \ltwob{\bv_m} \le \sqrt{2} R +  \sqrt{\frac{C R^2 \delta^{-1}}{ M}}  ~~~~~\text{ and }~~~~~ \sum_{m=1}^M \ltwob{\bv_m}^2 \le \frac{C R^2 \delta^{-1}}{ M}. \label{equ::coeff for finite neuron}
\end{equation}
\end{lemma}

\begin{lemma}\label{lem:RKHS_norm_target_function_Theorem1}
Under the setting of Theorem \ref{thm:finite-width}, let $f_\star$ be a target function of form \eqref{eqn:target-function} and \eqref{eqn:target-expansion}. Then there exists an infinite-width \RFA~model \eqref{eqn:infinite RF} with $\bv: \R^{(d+1 ) \times (d+1)} \to \R$ such that
\begin{equation}\label{eqn:fstar_representation_in_lemma}
f_\star = \frac{1}{N} \sum_{i=1}^N \E_{\bW} [\bar\sigma(\bW,\tbx_{0}\tbx_{i}^\top) \<\bv\paren{\bW},\tbx_i\>],
\end{equation}
with 
\[ \E_{\bW}[ \ltwos{\bv\paren{\bW}}^2] \le \rkhsbound,
\]
where $\rkhsbound$ is as defined in \eqref{eqn:bfstar}. %
\end{lemma}

The proofs of Lemma \ref{lem::rf and finite neuron} and \ref{lem:RKHS_norm_target_function_Theorem1} are given in Section \ref{sec:proof_aux_Theorem1}. Now we assume these two lemmas hold, and use them to prove Theorem \ref{thm:finite-width}. 

\begin{proof}[Proof of Theorem \ref{thm:finite-width}]

For any function that takes form \eqref{eqn:target-function} and \eqref{eqn:target-expansion}, by Lemma \ref{lem:RKHS_norm_target_function_Theorem1}, it admits representation \eqref{eqn:fstar_representation_in_lemma} with $\E_{\bW}[ \ltwos{\bv\paren{\bW}}^2] \le \rkhsbound$. Then by Lemma \ref{lem::rf and finite neuron}, since \RFA~model is a special case of the \RF~model, there exists $\sets{\bv_m}_{m\in[M]}$ such that Eq. \eqref{equ::finite neuron} and Eq. \eqref{equ::coeff for finite neuron} hold with probability larger than $1-\delta$. This proves Theorem \ref{thm:finite-width}.
\end{proof}

\subsubsection{Proof of auxiliary lemmas}\label{sec:proof_aux_Theorem1}

\begin{proof}[Proof of Lemma \ref{lem::rf and finite neuron}]~

\noindent
{\bf Step 1. Proof of Eq. \eqref{equ::finite neuron}. } To prove Eq. \eqref{equ::finite neuron}, we use a truncation argument. 

Fix a $R_{\bW} > 0$ which we will choose its value later in the proof.
Recall that we have $\bW_1,\ldots,\bW_m$ i.i.d. with $\bW_{m,i,j} \simiid \normal\paren{0,1/4}$. Define $\bv_m = {\bv}\paren{\bW_m} \indics{\lfros{\bW_m}\le R_{\bW}/2}/M$ for $m=1,\ldots,M$. Consider the truncated infinite-width random feature  model $f^{R_{\bW}}_{\bv} = \frac{1}{N} \sum_{i=1}^N \E_{\bW}[\bar\sigma(\bW,\tbX) \<\bv(\bW),\tbx_i\>\indics{\lfros{\bW}\le R_{\bW}/2}]$, we have
    \begin{align*}
        &~ \E_{\bW}\Big[\Big\| f^{R_{\bW}}_{\bv} - \frac{1}{N} \sum_{i=1}^N \sum_{m=1}^M \bar\sigma(\bW,\tbX) \<\bv_m,\tbx_i\> \Big\|_{L^2\paren{\Px}}^2\Big] \\
    \le &~ \E_{\bX}\Big( \E_{\bW} \Big[ \Big\|\frac{1}{N} \sum_{i=1}^N \bar\sigma(\bW,\tbX) \tbx_i\Big\|^2_2 \Big\|\bv\paren{\bW}\Big\|^2_2 \indic{\lfros{\bW}\le R_{\bW}/2} \Big] \Big)/M \\
    \le &~ \E_{\bX}\Big(   \E_{\bW} \Big[  \frac{2}{N} \sum_{i=1}^N \|\tbx_i\|^2_2 (L_1^2 \|\bW\|^2_2 + L_3^2) \|\bv(\bW)\|^2_2 \indic{\lfros{\bW}\le R_{\bW}/2}\Big]\Big)/M \\
    \le &~  \frac{\wt{C} R_{\bW}^2 R^2}{M},
    \end{align*}
    where $\wt{C_1}$ only depends on $L_1$, $L_2$, and $L_3$. \hengyu{fix C} Then using Markov's inequality,
    \begin{align*}
    \Big\| f_{\bv}^{R_{\bW}} - \frac{1}{N} \sum_{i=1}^N \sum_{m=1}^M \bar\sigma(\bW,\tbX) \<\bv_m,\tbx_i\> \Big\|_{L^2\paren{\Px}}^2 \le   \frac{3 \wt{C} R^2_{\bW}R^2}{\delta M}
    \end{align*}
    holds with probability at least $1-\delta/3$. Next for the difference between $f^{R_{\bW}}_{\bv}$ and $f_{\bv}$, we have
    \begin{align*}
        \norms{f_{\bv} - f^{R_{\bW}}_{\bv}}_{L^2\paren{\Px}}^2  \le&~ \E_{\bX} \Big[ \Big(\E_{\bW}\Big[ \frac{1}{N} \sum_{i=1}^N \bar\sigma(\bW,\tbX) \<\bv,\tbx_i\> \indic{\lfros{\bW}\ge R_{\bW}/2} \Big] \Big)^2 \Big] \\
        \le&~ \E_{\bX} \Big[ \frac{1}{N} \sum_{i=1}^N \E_\bW \Big[\<\bv,\tbx_i\>^2\Big] \sqrt{\E_\bW[\bar\sigma(\bW,\tbX)^4]} \sqrt{\P(\lfros{\bW} \ge R_{\bW}/2)}\Big] \\
         \le&~ \wt{C_2} R^2 \P\paren{\lfros{\bW} > R_{\bW}/2 }^{\frac{1}{2}}. 
    \end{align*}
    Here $\wt{C_2}$ is a constant that only depends on $L_1$ and $L_2$.
     By concentration of functions of Gaussian random vectors (Theorem \ref{thm:gaussian concentration}), $\|\bar\sigma(\bW,\tbX)\|_{\psi_2} \le L_1+L_2$ for any $i$. So in the last inequality, we used the bound $(\E_\bW[\bar\sigma(\bW,\tbX)^4])^{1/2}$ by $\Theta((L_1+L_2)^2)$. To bound $\P(\lfros{\bW} > R_{\bW}/2)$, we use concentration of functions of Gaussian random vectors (Theorem \ref{thm:gaussian concentration}) again, and get that
    \begin{equation*}
        \P\big( \lfros{\bW} - \E(\lfros{\bW}) \ge t/2 \big) \le \exp\paren{-t^2/2}.
    \end{equation*}
    Take $C=max(\wt{C_1},\wt{C_2})$. Since $\E(\lfros{\bW}) \le (\E\lfros{\bW}^2)^{1/2} \le d+1$, by choosing $R_{\bW} = d + 1 + C\sqrt{\log M}$, the above probability is less than $1/M^2$. 
    Then 
    \begin{align*}
        \Big\| f_{\bv} - \frac{1}{N} \sum_{i=1}^N \sum_{m=1}^M \bar\sigma(\bW,\tbX) \<\bv_m,\tbx_i\> \Big\|_{L^2\paren{\Px}}^2 
        \le &~ \frac{3CR^2_{\bW}R^2}{\delta M} + C R^2 \P\paren{\lfros{\bW} > R_{\bW}/2 }^{\frac{1}{2}} \\
        \le &~ \frac{C (\log M + d^2) R^2}{\delta M}
    \end{align*}
    with probability larger than $1-\delta/3$. This proves Eq. \eqref{equ::finite neuron}. 
    
\noindent
{\bf Step 2. Proof of Eq. \eqref{equ::coeff for finite neuron}. } By Chebyshev's inequality, 
    \begin{align*}
        &~ \P\paren{  \sum_{m=1}^M \ltwo{\bv_m} - \E\bigbrac{\ltwo{\bv(\bW)} \indic{\lfros{\bW_m}\le R_{\bW}/2} } \ge \sqrt{\frac{6R^2}{\delta M}}}\\
        \le &~ \frac{\E\bigbrac{\ltwo{\bv(\bW)}^2\indic{\lfros{\bW}\le R_{\bW}/2}}\delta M}{6R^2 M} \le \frac{\delta}{3}.
    \end{align*}
    Combining with the fact that $\E \bigbrac{\ltwo{\bv(\bW)}\indic{\lfros{\bW}\le R_{\bW}/2}} \le \sqrt{\E \bigbrac{\ltwo{\bv(\bW)}^2} } \le R$, we have
    \[ \P\paren{ \sum_{m=1}^M \ltwo{\bv_m} \ge R + \sqrt{\frac{6R^2}{\delta M}} } \le \frac{\delta}{3}. \]
    For the second part of (\ref{equ::coeff for finite neuron}), Markov inequality gives
    \[\P\paren{ \sum_{m=1}^M \ltwo{\bv_m}^2 \ge \frac{6R^2}{\delta M} }\le \frac{\E\bigbrac{\ltwo{\bv(\bW)}^2\indic{\lfros{\bW_m}\le R_{\bW}/2}} \delta M}{6R^2 M} \le \frac{\delta}{3}.\]
    This proves Eq. \eqref{equ::coeff for finite neuron} and completes the proof of Lemma \ref{lem::rf and finite neuron}. 
\end{proof}

\begin{proof}[Proof of Lemma \ref{lem:RKHS_norm_target_function_Theorem1}]~

To get the kernel of the \RFA~model, we have 
\[ K\paren{\sequenceX{},\sequenceX{\prime}} = \frac{1}{N^2} \E_{\bW}\Big[ \sum_{i,j=1}^N \sigma(\langle\bW,\tbX\rangle) \sigma(\langle\bW,\tbXp\rangle) \langle \tbx_i,\tbx_j^{\prime}\rangle \Big].\]
We first consider a single component in the sum, which is
\begin{equation}\label{eq:Gaussian single comp}
    \E_{\bW}\bigbrac{\barsig\paren{\langle \bW,\tbX\rangle}\barsig\paren{\langle\bW,\tbx_0^\prime(\tbx_j)^\prime\rangle} }\langle \tbx_i, \tbx_j'\rangle.
\end{equation}
Let $u_{i,j}=\<\tbX,\tbXp\> / 4$. Let $\normal_2(\rho)$ denote a bivariate normal distribution with marginals are $\normal\paren{0,1}$ and the correlation is $\rho \in [-1,1]$.
Then \eqref{eq:Gaussian single comp} can be expanded as follows:
\begin{align*}
    &~\E_{\bW}\bigbrac{\barsig\paren{\langle \bW,\tbX\rangle}\barsig\paren{\langle\bW,\tbXp\rangle} }\langle \tbx_i, \tbx_j^\prime\rangle\\
    = &~ \E_{Z_1,Z_2 \sim\normal_2(u_{i,j})}\bigbrac{\barsig\paren{Z_1}\barsig\paren{Z_2} }\langle \tbx_i, \tbx_j^\prime\rangle \\
    = &~ \frac{1}{2\pi}\paren{u_{i,j} \paren{\pi/2 - \arccos u_{i,j}}
     + \sqrt{1 - u_{i,j}^2}}\langle \tbx_i, \tbx_j^\prime\rangle \\
    = &~ \frac{1}{2\pi}
     \paren{1 + \frac{\pi}{2} u_{i,j} +
     \sum_{\ell=1}^\infty 
     \frac{(2\ell-3)!!}{(2\ell)!!(2\ell-1)} u_{i,j}^{2\ell}}\langle \tbx_i, \tbx_j^\prime\rangle \\
    = &~ \sum_{\ell\in\set{0,1}\cup \set{2k}_{k\ge
    1}} c_\ell \<\tbX,\tbXp\>^\ell 4^{-\ell} \langle \tbx_i, \tbx_j^\prime\rangle \\
    = &~ \sum_{\ell\in\set{0,1}\cup \set{2k}_{k\ge
    1}} c_\ell \Big\langle 2^{-\ell} (\tbX)^{\otimes \ell}\otimes \tbx_i,  2^{-\ell} (\tbXp)^{\otimes \ell}\otimes \tbx_j^\prime\Big\rangle.
\end{align*}
Here the coefficients $\set{c_\ell}$ satisfy
  \begin{equation*}
    c_0 = 1/(2\pi), ~~~c_1 = 1/4, ~~~\text{and}~~~c_{2\ell} =
    \frac{(2\ell-3)!!}{2\pi(2\ell)!!(2\ell-1)} =O(\ell^{-\frac{5}{2}}) ~~{\rm for}~\ell  \ge 1.
  \end{equation*}
Therefore, the kernel can be expressed as:
\begin{align*}
    &~K\paren{\sequenceX{},\sequenceX{\prime}}\\ 
    =&~ \frac{1}{N^2}\sum_{1\leq i,j \leq N}\sum_{\ell\in\set{0,1}\cup \set{2k}_{k\ge
    1}} c_\ell \Big\langle 2^{-\ell} (\tbX)^{\otimes \ell}\otimes \tbx_i,  2^{-\ell} (\tbXp)^{\otimes \ell}\otimes \tbx_j^\prime\Big\rangle \\
    =&~\sum_{\ell\in\set{0,1}\cup \set{2k}_{k\ge
    1}}  \Big\langle \frac{\sqrt{c_\ell}}{N} \sum_{i=1}^N 2^{-\ell} (\tbX)^{\otimes \ell}\otimes \tbx_i,\frac{\sqrt{c_\ell}}{N} \sum_{j=1}^N  2^{-\ell} (\tbXp)^{\otimes \ell}\otimes \tbx_j^\prime  \Big\rangle.
\end{align*}
Now we reformulate the target function as
\begin{align*}
    f_\star = &~ \frac{1}{N} \sum_{i=1}^N \sum_{\ell=0}^\infty \sum_{\max\{r,s\} = \ell}  \<\tbx_0^{\otimes r} \otimes \tbx_i^{\otimes s}, \mathbf{f}_{rs}\> \\
    = &~ \frac{1}{N}\sum_{i=1}^N \sum_{\ell=2k,k\ge 0}^\infty \sum_{\max\{r,s\} = \ell\text{ or }\ell-1}  \<(\tbx_0 \otimes \tbx_i)^{\otimes \ell} \otimes \tbx_i, \wt{\mathbf{f}}_{rs}\>\\
     = &~ \sum_{\ell = 2k,k\ge 0}^\infty \Big\langle \frac{2^{\ell}}{\sqrt{c_\ell}}  \sum_{\max\{r,s\} = \ell\text{ or }\ell-1} \wt{\mathbf{f}}_{rs}, \frac{\sqrt{c_\ell}}{N} \sum_{i=1}^N 2^{-\ell} (\tbX)^{\otimes \ell}\otimes \tbx_i \Big\rangle,
\end{align*}
where $\wt{\mathbf{f}}_{rs}$ is a transpose of $\mathbf{f}_{rs}\otimes \mathbf{1}_{d+1}^{\otimes (2\ell-r-s+1)}$ with $\mathbf{1}_{d+1} = (0,\ldots,0,1)$, such that $\<\tbx_0^{\otimes r} \otimes \tbx_i^{\otimes s}, \mathbf{f}_{rs}\> = \langle(\tbx_0 \otimes \tbx_i)^{\otimes \ell} \otimes \tbx_i, \wt{\mathbf{f}}_{rs}\rangle$ for any $r\ge 0$ and $s\ge 0$. Then by the feature map equivalence property \eqref{eqn:feature map equi}, the RKHS norm of $f^\star$ can be bounded as %
\begin{align*}
    \norm{f^\star}_{\mc{H}_K}^2  \le &~ \Big\|\sum_{\ell = 2k,k\ge 0}^\infty \Big\langle \frac{2^{\ell}}{\sqrt{c_\ell}} \sum_{\max\{r,s\} = \ell\text{ or }\ell-1} \wt{\mathbf{f}}_{rs}, \frac{\sqrt{c_\ell}}{N} \sum_{i=1}^N 2^{-\ell} (\tbX)^{\otimes \ell}\otimes \tbx_i \Big\rangle\Big\|_{\mc{H}_K}^2 \\
    =&~\sum_{\ell = 2k,k\ge 0}^\infty \Big\langle \frac{2^{\ell}}{\sqrt{c_\ell}} \sum_{\max\{r,s\} = \ell\text{ or }\ell-1} \wt{\mathbf{f}}_{rs}, \frac{2^{\ell}}{\sqrt{c_\ell}}  \sum_{\max\{r,s\} = \ell\text{ or }\ell-1} \wt{\mathbf{f}}_{rs} \Big\rangle \\
    =&~\sum_{\ell = 2k,k\ge 0}^\infty 4^{\ell} c_\ell^{-1} \Big\|\sum_{\max\{r,s\} = \ell\text{ or }\ell-1} \wt{\mathbf{f}}_{rs}\Big\|_{\rm Fr}^2 \\
    \le&~\sum_{k=0}^\infty 4^k k^{4.5} \sum_{\max\{r,s\}=k} \lfro{\wt{\mathbf{f}}_{rs}}^2.
\end{align*}
Thus, using again the property \eqref{eqn:feature map equi} with the original feature map of the random feature model, there exists $\bv: \R^{(d+1 ) \times (d+1)} \to \R^{d+1}$ such that
\begin{align*}
f_\star =&~ \frac{1}{N} \sum_{i=1}^N \E_{\bW} [\bar\sigma(\bW,\tbx_{0}\tbx_{i}^\top) \<\bv\paren{\bW},\tbx_i\>],~~\text{with}\\
\E_{\bW}[ \ltwos{\bv\paren{\bW}}^2] \le&~ \sum_{k=0}^\infty 4^k k^{4.5} \sum_{\max\{r,s\}=k} \lfros{\wt{\mathbf{f}}_{rs}}^2.
\end{align*}
Notice that $\lfros{\wt{\mathbf{f}}_{rs}}^2 = \lfros{{\mathbf{f}}_{rs}}^2$ by our construction of $\wt{\mathbf{f}}_{rs}$, so that the right-hand-side of the equation above coincides with Eq. \eqref{eqn:bfstar}. This proves Lemma \ref{lem:RKHS_norm_target_function_Theorem1}. 
\end{proof}

\subsection{Preliminary proposition for Theorem \ref{thm::sample complexity}}
\label{sec::proof of sample complexity}

To prove Theorem \ref{thm::sample complexity}, we first present and prove the following proposition that gives a high probability bound for the difference between the empirical risk and the population risk. In the proposition and lemmas below, we denote $\bX = \{ \sequenceX{(j)} \}_{j \in [n]}$ and $\by = \set{y_j}_{j \in [n]}$. 
\begin{proposition}
\label{prop::empirical process}
    Under the setting of Theorem \ref{thm::sample complexity}. Consider the finite width \RF~model \eqref{eqn:finite RF}:
\begin{align}
    \tfwm(\bx_{0:N}; \bV) = \sum_{m=1}^M \frac{1}{N}\sum_{j=1}^n \bar\sigma\paren{\bW_m, \tbX} \<\bv_m, \tbx_i\>. \nonumber
\end{align}
Then with probability at least $1-\delta$ (w.r.t. $\bW,$ $\by$, and $\bX$), we have
    \begin{align}
 &~\sup_{\bV \in \mathcal{V}_M} \bigabs{ \frac{1}{n} \sum_{j=1}^n \ell\paren{\tfwm\paren{\sequenceX{(j)};\bV},y_j} - \E_{\sequenceX{},y}\ell\paren{\tfwm(\sequenceX{};\bV),y}}\nonumber\\ 
 \lesssim &~ \Bvone \sqrt{\frac{\log (dM) \log \paren{nNM}}{n}} + \sqrt{\log\paren{\frac{6}{\delta}}}\paren{\frac{\Bvone}{\sqrt{n}} + \sqrt{\frac{\Bvtwo}{M}}}.\label{equ::empirical process}
    \end{align}
\end{proposition}

The main difficulty of the proof of Proposition \ref{prop::empirical process} comes from that $\ell\paren{\tfwm,y}$ might be unbounded and that $\ell(\tfwm(\sequenceX{(j)}),y_j)$ are not independent across $j$ (since they share the same $\sets{\bW_m}_{m\in[M]}$). So we begin with several lemmas below. 
\begin{lemma}
\label{lem::emp process mean with E xyw}
Let $\sets{\xi_j}_{j\in [n]}$ be a set of $i.i.d.$ Rademacher random variables. Under the setting of Proposition~\ref{prop::empirical process}, 
\begin{equation}
    \E_{\bX,\by,\bW,\boldsymbol{\xi}} \Big[\sup_{\bV \in \fnclass} \Big| \dn \sumn \xi_j \ell\paren{\tfwm\paren{\sequenceX{(j)}; \bV},y_j} \Big|\Big] \lesssim \Bvone \sqrt{\frac{\log (dM) \log \paren{nNM}}{n}}.
    \label{eqn:emp process mean with E xyw form 1}
\end{equation}
Furthermore, any fixed $\bX$ and $\by$,
\begin{equation}
    \E_{\bW,\boldsymbol{\xi}} \Big[\sup_{\bV \in \fnclass} \Big| \dn \sumn \xi_j \ell\paren{\tfwm\paren{\sequenceX{(j)}; \bV},y_j} \Big|\Big] \lesssim \Bvone \sqrt{\frac{\log (dM) \log \paren{nNM}}{n}}.
        \label{eqn:emp process mean with E xyw form 2}
\end{equation}
\end{lemma}
\begin{lemma}
\label{lem::emp process with E w}
Under the setting of Proposition~\ref{prop::empirical process}. With probability at least $1-\delta/3$ over $\bX$, $\by$, and $\bW$,
\begin{align}
    &~ \sup_{\bV \in \fnclass} \Big|\dn \sumn \ell\paren{\tfwm\paren{\sequenceX{(j)}; \bV},y_j} - \E_{\bW} \Big[\dn \sumn \ell\paren{\tfwm\paren{\sequenceX{(j)}; \bV},y_j}\Big]\Big| \nonumber \\
    \lesssim &~ \sqrt{\frac{\Bvtwo\log(6/\delta)}{M}} + \Bvone \sqrt{\frac{\log (dM) \log \paren{nNM} }{n}}. \label{equ::emp process with E w}
\end{align}
\end{lemma}
\begin{lemma}
\label{lem::emp process E w with E xyw}
Under the setting of Proposition~\ref{prop::empirical process}. With probability at least $1-\delta/3$ over $\bX$ and $\by$,
    \begin{align}
        &~ \sup_{\bV \in \fnclass} \Big| \E_{\bW} \Big[ \dn \sumn \ell\paren{\tfwm\paren{\sequenceX{(j)}; \bV},y_j}\Big] - \E_{\bX,\by,\bW}\Big[\dn \sumn \ell\paren{\tfwm\paren{\sequenceX{(j)}; \bV},y_j}\Big]\Big| \nonumber \\ 
        \lesssim &~ \Bvone \sqrt{\frac{\log (dM) \log \paren{nNM} }{n}} + \Bvone \sqrt{ \frac{\log\paren{6/\delta}}{n}}. 
    \end{align}
\end{lemma}
\begin{lemma}
\label{lem::emp process E xy with E xyw}
Under the setting of Proposition~\ref{prop::empirical process}. With probability at least $1-\delta/3$ over $\bW$,
\begin{align}
        &~ \sup_{\bV \in \fnclass} \Big| \E_{\bX,\by} \Big[ \dn \sumn \ell\paren{\tfwm\paren{\sequenceX{(j)}; \bV},y_j}\Big] - \E_{\bX,\by,\bW}\Big[\dn \sumn \ell\paren{\tfwm\paren{\sequenceX{(j)}; \bV},y_j}\Big]\Big|  \nonumber \\ 
        \lesssim &~ \sqrt{\frac{\Bvtwo\log(6/\delta)}{M}} + \Bvone \sqrt{\frac{\log (dM) \log \paren{nNM} }{n}}. 
    \end{align}
\end{lemma}
The proofs of Lemma \ref{lem::emp process mean with E xyw}, \ref{lem::emp process with E w}, \ref{lem::emp process E w with E xyw}, and \ref{lem::emp process E xy with E xyw} are contained in section \ref{prf:proof of lemmas for thm 2}. Now  assuming they hold, we proceed to prove Proposition~\ref{prop::empirical process}.

\begin{proof}[Proof of Proposition \ref{prop::empirical process}]~

Split the left-hand side of inequality~\eqref{equ::empirical process}, we have
\begin{align*}
    &~  \sup_{\bV \in \mathcal{V}_M} \Big| \frac{1}{n} \sum_{j=1}^n \ell\Big(\tfwm\paren{\sequenceX{(j)}; \bV},y_j\Big) - \E_{\bX, \by}\Big(\ell\paren{\tfwm\paren{\sequenceX{(j)}; \bV},y}\Big)\Big| \\
\le &~ \sup_{\bV \in \fnclass} \Big|\dn \sumn \ell\paren{\tfwm,y_j} - \E_{\bW} \Big[\dn \sumn \ell\paren{\tfwm,y_j}\Big] \Big| \\
 &~ + \sup_{\bV \in \fnclass} \Big| \E_{\bW} \Big[\dn \sumn \ell\paren{\tfwm,y_j}\Big] - \E_{\bX, \by,\bW}\Big( \ell\paren{\tfwm,y} \Big) \Big|\\
 &~ + \sup_{\bV \in \fnclass} \bigabs{ \E_{\bX, \by} \Big( \ell\paren{\tfwm,y}\Big) - \E_{\bX, \by,\bW} \Big(\ell\paren{\tfwm,y}\Big) }\\
\lesssim &~ \Bvone \sqrt{\frac{\log (dM) \log \paren{nNM} }{n}} + \sqrt{\log\paren{\frac{6}{\delta}}}\paren{\frac{\Bvone}{\sqrt{n}} + \sqrt{\frac{\Bvtwo}{M}}}
\end{align*}
with probability at least $1-\delta$. Here the last inequality uses Lemma \ref{lem::emp process with E w}, \ref{lem::emp process E w with E xyw}, and \ref{lem::emp process E xy with E xyw}. This proves Proposition~\ref{prop::empirical process}.
\end{proof}

\subsubsection{Proof of auxiliary lemmas}
\label{prf:proof of lemmas for thm 2}
\begin{proof}[Proof of Lemma~\ref{lem::emp process mean with E xyw}]~

    First using Rademacher contraction inequality, since $\ell(0,y)\le 1$, we can center it and only pay an extra term $1/\sqrt{n}$ in the Rademacher complexity. Then by the Rademacher contraction property (Theorem \ref{thm:rademacher contraction}), the problem boils down to bounding the Rademacher complexity of $\tfwm$, which is
    \[ \E_{\bX,\by,\bW,\boldsymbol{\xi}} \bigbrac{ \sup_{\bV \in \fnclass} \bigabs{\frac{1}{n} \sum_{j=1}^n \xi_j\tfwm\paren{\sequenceX{(j)}; \bV}}}.\]
    Fix $\bX$, $\by$, and $\bW$, we have 
    \begin{align}
         &~ \E_{\boldsymbol{\xi}} \bigbrac{ \sup_{\bV \in \fnclass} \bigabs{\frac{1}{n} \sum_{j=1}^n \xi_j\tfwm\paren{\sequenceX{(j)}; \bV}}} \nonumber \\
        = &~ \E_{\boldsymbol{\xi}} \bigbrac{ \sup_{\bV \in \fnclass } \bigabs{   \sum_{m=1}^M \Big\langle \bv_m,\dn \sumn \xi_j \bigbrac{\frac{1}{N} \sum_{i=1}^N \bar\sigma(\bW_m,\tbx_0^{(j)}\tbx_i^{(j)\top}) \tbx_i^{(j)} } \Big\rangle  } } \nonumber \\
        \le &~ \Bvone \E_{\boldsymbol{\xi}}\bigbrac{\max_m \Big\| \dn \sumn \xi_j \bigbrac{\frac{1}{N} \sum_{i=1}^N \bar\sigma(\bW_m,\tbx_0^{(j)}\tbx_i^{(j)\top}) \tbx_i^{(j)} } \Big\|_2 }. \label{equ::maximum of sub exp} 
    \end{align}
    By matrix Bernstein inequality (Theorem \ref{thm:matrix bernstein}), for any fixed $m$, 
    \begin{align*}
        \P\bigbrac{\Big\| \dn \sumn \xi_j \bigbrac{\frac{1}{N} \sum_{i=1}^N \bar\sigma(\bW_m,\tbx_0^{(j)}\tbx_i^{(j)\top}) \tbx_i^{(j)}}\Big\|_2 \ge \epsilon} \le 2d\exp\paren{ -\frac{n \epsilon^2/2}{A^2 + K\epsilon/3} },
    \end{align*}    
    where
    \begin{align*}
       A = &~ \max_m \sqrt{\frac{1}{nN^2} \sum_{i,j,k} \bigbrac{\bar\sigma(\bW_m,\tbx_0^{(j)}\tbx_i^{(j)\top}) \bar\sigma(\bW_m,\tbx_0^{(j)}\tbx_k^{(j)}) \< \tbx_i^{(j)},\tbx_k^{(j)}\>}} \\
       \lesssim &~ \max_{m,i,j} \bar\sigma(\bW_m,\tbx_0^{(j)}\tbx_i^{(j)\top}), \text{ and}\\
        K = &~ \max_{i,m} \Big\| \frac{1}{N} \sum_{i=1}^N \bar\sigma(\bW_m,\tbx_0^{(j)}\tbx_i^{(j)\top}) \tbx_i^{(j)} \Big\|_2.
    \end{align*}
    Using the union bound, we have
    \begin{align*}
    \P\bigbrac{\max_{m\in[M]}\ltwob{ \dn \sumn \xi_j \bigbrac{\frac{1}{N} \sum_{i=1}^N \bar\sigma(\bW_m,\tbx_0^{(j)}\tbx_i^{(j)\top}) \tbx_i^{(j)}}} \ge \epsilon} \le 2dM\exp\paren{-\frac{n\epsilon^2}{A^2 + K\epsilon/3}}.
    \end{align*}
    Therefore, we can bound its expectation with
    \begin{align}
    \E_{\boldsymbol{\xi}}\bigbrac{\max_m \ltwob{ \dn \sumn \xi_j \bigbrac{\frac{1}{N} \sum_{i=1}^N \bar\sigma(\bW_m,\tbx_0^{(j)}\tbx_i^{(j)\top}) \tbx_i^{(j)}}}}
    \lesssim  \bigbrac{ \sqrt{\frac{\log (dM)}{n}} A + \frac{\log (dM)}{n} K }.\label{eqn:bound_complex_object_by_A+K}
    \end{align}
    Now take expectation over $\bX$, $\by$, and $\bW$. Since $n \ge \log(d M)$, we have
    \begin{align*}
     &~   \E_{\bX,\by,\bW}\bigbrac{ \sqrt{\frac{\log (dM)}{n}} A + \frac{\log (dM)}{n} K }\\
    \lesssim &~ \paren{\sqrt{\frac{\log (dM)}{n}} + \frac{\log (dM)}{n}} \E_{\bX,\by,\bW} \bigbrac{ \max_{i,j,m} \bigabs{\bar\sigma\paren{\<\bW_m,\tbx_0^{(j)}\tbx_i^{(j)\top}\>}}} \\
    \lesssim &~  \sqrt{\frac{\log (dM)\log \paren{nNM} }{n}}.
    \end{align*}   
    Combine this with Eq. \eqref{equ::maximum of sub exp} and \eqref{eqn:bound_complex_object_by_A+K}, we prove \eqref{eqn:emp process mean with E xyw form 1}.
    
    Fixing any $\bX$ and $\by$, only taking expectation over $\bW$, we get 
    \begin{align*}
   \E_{\bW}\bigbrac{ \sqrt{\frac{\log (dM)}{n}} A + \frac{\log (dM)}{n} K } \lesssim \sqrt{\frac{\log (dM)\log \paren{nNM} }{n}}.
    \end{align*}  
    Combine this with Eq. \eqref{equ::maximum of sub exp} and \eqref{eqn:bound_complex_object_by_A+K}, we prove \eqref{eqn:emp process mean with E xyw form 2}. This finishes the proof of Lemma \ref{lem::emp process mean with E xyw}.
\end{proof}
\begin{proof}[Proof for Lemma~\ref{lem::emp process with E w}]~

    Denote $\bX = \{ \sequenceX{(j)} \}_{j \in [n]}$ and $\by = \set{y_j}_{j \in [n]}$ and denote
    \[g\paren{\bW_{1:M};\bX,\by} = \sup_{\bV \in \fnclass} \bigabs{ \dn \sumn \ell\paren{\tfwm\paren{\sequenceX{(j)}; \bV},y_j} - \E_{\bW} \bigbrac{ \dn \sumn \ell\paren{\tfwm\paren{\sequenceX{(j)}; \bV},y_j}}}.\]
    Given $\bW_{1:M}=\set{\bW_m}_{m=1}^M$ and $\bW_{1:M}^\prime=\set{\bW_m^\prime}_{m=1}^M$, we define $\lfros{\bW_{1:M} - \bW_{1:M}^\prime} = ({\sum_{m\in[M]} \lfros{\bW_m - \bW^\prime_m}^2})^{1/2}$. Then we have 
    \begin{align*}
        &~ g\paren{\bW_{1:M}; \bX,\by} - g\paren{\bW_{1:M}^\prime;\bX,\by} \\
        \le &~ \sup_{\bV \in \fnclass} \bigabs{\dn \sumn \bigbrac{\ell\paren{\tfwm,y_j} - \ell\paren{f^{\bW^\prime}_M,y_j}}} \\
        \le &~ \sup_{\bV \in \fnclass} \bigabs{\frac{1}{Nn} \sum_{i,j,m} \bigbrac{\bar\sigma(\bW,\tbx_0^{(j)}\tbx_i^{(j)\top}) - \bar\sigma(\bW^\prime,\tbx_0^{(j)}\tbx_i^{(j)\top}) } \<\bv_m,\tbx_i\>} \\
        \lesssim &~ \sup_{\bV \in \fnclass} \bigabs{   \sum_m \lfros{\bW_m - \bW_m^\prime} \ltwo{\bv_m} } \\
        \lesssim &~ \lfros{ \bW_{1:M} - \bW'_{1:M}} \sqrt{\frac{\Bvtwo}{M}}.
    \end{align*}
    Since $\bW_{m}$ has independent standard Gaussian entries, by Gaussian concentration inequality (Theorem \ref{thm:gaussian concentration}), we have that
    \begin{equation}
        \P\bigbrac{\bigabs{g\paren{\bW_{1:M};\bX,\by} - \E_{\bW} \Big[g\paren{\bW_{1:M};\bX,\by} \mid \bX,\by  \Big]} \ge \epsilon \mid \bX,\by} \le 2\exp\paren{-\frac{M \epsilon^2}{4 \Bvtwo}}. \label{equ::emp process with E w concentration}
    \end{equation}
    For the conditional expectation, by symmetrization, we have
    \begin{align}
     &~\E_{\bW} \bigbrac{g\paren{\bW_{1:M};\bX,\by}\mid \bX,\by} \nonumber\\
     = &~ \E_{\bW} \bigbrac{ \sup_{\bV \in \fnclass} \bigabs{ \sum_{j=1}^n \ell\paren{\tfwm\paren{\sequenceX{(j)}; \bV},y_j} - \E_{\bW}\Big[\ell\paren{\tfwm\paren{\sequenceX{(j)}; \bV},y_j}\Big] } \mid \bX,\by} \nonumber\\
     \le &~ 2\E_{\bW,\boldsymbol{\xi}} \bigbrac{ \sup_{\bV \in \fnclass} \bigabs{ \sum_{j=1}^n \xi_j \ell\paren{\tfwm\paren{\sequenceX{(j)}; \bV},y_j}} \mid \bX,\by} \nonumber\\
    \lesssim &~ \Bvone \sqrt{\frac{\log (dM) \log \paren{nNM}}{n}} \label{equ::emp process with E w mean part}
    \end{align} 
    for any $\bX$ and $\by$, where the last inequality uses Lemma \ref{lem::emp process mean with E xyw}. Combining (\ref{equ::emp process with E w concentration}) and (\ref{equ::emp process with E w mean part}) and taking $\epsilon = 2 \sqrt{{\Bvtwo}\log\paren{6/\delta}/M}$, we have 
    \begin{align}
    &~\sup_{\bV \in \fnclass} \bigabs{\dn \sumn \ell\paren{\tfwm\paren{\sequenceX{(j)}; \bV},y_j} - \E_{\bW} \bigbrac{\dn \sumn \ell\paren{\tfwm\paren{\sequenceX{(j)}; \bV},y_j}}} \nonumber \\
    \lesssim &~\sqrt{\frac{\Bvtwo}{M}}\sqrt{\log\frac{6}{\delta}} + \Bvone \sqrt{\frac{\log (dM) \log \paren{nNM}}{n}}, 
\end{align}
    for any $\bX$ and $\by$. Since the right-hand side is irrelevant to $\bX$ and $\by$, we get (\ref{equ::emp process with E w}). This proves Lemma \ref{lem::emp process with E w}.
\end{proof}
\begin{proof}[Proof of Lemma~\ref{lem::emp process E w with E xyw}]~

    Let 
    \begin{align}
    h\paren{\sequenceX{},\by} =  \sup_{\bV \in \fnclass} \bigabs{ \E_{\bW} \bigbrac{ \dn \sumn \ell\paren{\tfwm,y_j}} - \E_{\bX,\by,\bW}\bigbrac{\dn \sumn \ell\paren{\tfwm,y_j}} }. \label{equ::fix x center diff} 
    \end{align}
    For each $i \in [n]$, let $\set{\bX',\by^\prime}$ differs with $\set{\bX,\by}$ only on $i$-th data point. We have
    \begin{align*}
    &~ h\paren{\bX,\by} - h\paren{\bX^\prime,\by^\prime} \\
        \le &~ \sup_{\bV \in \fnclass} \bigabs{\dn \sumn \bigbrac{\E_{\bW} \bigbrac{\ell\paren{\tfwm\paren{\sequenceX{(j)}; \bV},y_j}} - \E_{\bW} \bigbrac{\ell\paren{\tfwm\paren{\sequenceX{(j)\prime}; \bV},y_j^\prime}}}} \\
        \le &~\sup_{\bV \in \fnclass} \E_{\bW} \bigbrac{\dn \sumn \bigabs{\tfwm\paren{\sequenceX{(j)}; \bV} - \tfwm\paren{\sequenceX{(j)\prime}; \bV}} }  \\
        \le &~ \sup_{\bV \in \fnclass} \E_{\bW} \bigbrac{\frac{1}{nN} \sum_{m \in [M], k \in [N]} \bigbrac{\bigabs{\bar\sigma(\bW_m,\tbx_0^{(i)}\tbx_k^{(i)\top})}+\bigabs{\bar\sigma(\bW_m,\tbx_0^{(i)\prime}\tbx_k^{(i)\prime\top} )}} \ltwo{\bv_m}}.\\
        \lesssim &~ \frac{\Bvone}{n}.
    \end{align*}
    Therefore, $h\paren{\sequenceX{},\by}$ satisfies the bounded difference property with the parameter $\set{L_i}_{i=1}^n$ uniformly bounded by $\Theta(\Bvone/n)$. By bounded difference inequality (Theorem \ref{thm:bounded difference}), there's a constant $\wt{C}$ such that
    \[ \P\bigbrac{ \bigabs{h\paren{\bX,\by} - \E_{\bX} \bigbrac{h\paren{\bX,\by}}} \ge \epsilon } \le 2\exp\paren{-\frac{\wt{C} n \epsilon^2}{\Bvone^2}}. \]
    Combining with Lemma~\ref{lem::emp process mean with E xyw}, we have
    \begin{align*}
     &~\E_{\bX,\by} \bigbrac{h\paren{\bX,\by}} \\
     = &~ \E_{\bX,\by} \bigbrac{ \sup_{\bV \in \fnclass} \bigabs{ \E_{\bW} \bigbrac{ \dn \sumn \ell\paren{\tfwm,y_j}} - \E_{\bX,\by,\bW}\bigbrac{\dn \sumn \ell\paren{\tfwm,y_j}} }} \\
     \le &~ 2 \E_{\bX,\by,\boldsymbol{\xi}} \bigbrac{ \sup_{\bV \in \fnclass} \bigabs{ \dn \sum_{j=1}^n \xi_j \E_{\bW}\bigbrac{\ell\paren{\tfwm\paren{\sequenceX{(j)}; \bV},y_j}}} }\\
     \lesssim &~ \Bvone \sqrt{\frac{\log (dM)\log \paren{nNM} }{n}}.
    \end{align*}
    Therefore, by taking $\epsilon = 2\Bvone [\log\paren{6/\delta}/(n\wt{C})]^{1/2}$, we have
    \begin{align*}
        &~ \sup_{\bV \in \fnclass} \bigabs{ \E_{\bW} \bigbrac{ \dn \sumn \ell\paren{\tfwm,y_j}} - \E_{\bX,\by,\bW}\bigbrac{\dn \sumn \ell\paren{\tfwm,y_j}}}\\ 
        \lesssim &~ \Bvone \sqrt{\frac{\log (dM) \log \paren{nNM} }{n}} + \Bvone \sqrt{\frac{\log\paren{1/\delta}}{n}}
    \end{align*}
    with probability at least $1-\delta$. This proves Lemma \ref{lem::emp process E w with E xyw}.
    \end{proof}
    
\begin{proof}[Proof of Lemma~\ref{lem::emp process E xy with E xyw}]~

    Denote
    \[\varphi\paren{\bW_{1:M}} = \sup_{\bV \in \fnclass} \bigabs{ \E_{\bX,\by}\bigbrac{\dn \sumn \ell\paren{\tfwm,y_j}} - \E_{\bX,\by,\bW} \bigbrac{ \dn \sumn \ell\paren{\tfwm,y_j}}}.\]
    Similar to the proof of Lemma~\ref{lem::emp process with E w}.
    Given $\bW_{1:M}=\set{\bW_m}_{m=1}^M$ and $\bW_{1:M}^\prime=\set{\bW_m^\prime}_{m=1}^M$, define $\lfros{\bW_{1:M} - \bW_{1:M}^\prime} = \sqrt{\sum_m \lfros{\bW_m - \bW^\prime_m}^2}$. We have
    \begin{align*}
        &~ \varphi\paren{\bW_{1:M}} - \varphi\paren{\bW_{1:M}^\prime} \\
        \le &~ \sup_{\bV \in \fnclass} \bigabs{\dn \sumn \E_{\bX,\by}\bigbrac{\ell\paren{\tfwm,y_j} - \ell\paren{f^{\bW^\prime}_M,y_j}} }\\
        \le &~ \sup_{\bV \in \fnclass} \bigabs{\frac{1}{Nn} \sum_{i,j,m} \E_{\bX,\by}\left\{\bigbrac{\bar\sigma(\bW,\tbx_0^{(j)}\tbx_i^{(j)\top}) - \bar\sigma(\bW^\prime,\tbx_0^{(j)}\tbx_i^{(j)\top}) } \<\bv_m,\tbx_i\>\right\}} \\
        \le &~ \sup_{\bV \in \fnclass} \bigabs{   \sum_m 2\sqrt{2} \lfros{\bW_m - \bW_m^\prime} \ltwos{\bv_m} } \\
        \lesssim &~ \lfros{ \bW_{1:M} - \bW_{1:M}} \sqrt{\frac{\Bvtwo}{M}}.
    \end{align*}
    Since $\bW_{m}$ has independent standard Gaussian entries, there is a constant $\wt{C}$ s.t.
    \begin{equation}
        \P\bigbrac{\bigabs{\varphi\paren{\bW_{1:M}} - \E_{\bW} \bigbrac{\varphi \paren{\bW_{1:M}}}} \ge \epsilon} \le 2\exp\paren{-\frac{\wt{C}M\epsilon^2}{\Bvtwo}}. \label{equ::emp process E xy with E xyw concentration}
    \end{equation}
    For the expectation, we have
    \begin{align}
      \E_{\bW} \bigbrac{\varphi\paren{\bW_{1:M}}}
     = &~ \E_{\bW} \bigbrac{ \sup_{\bV \in \fnclass} \bigabs{ \dn \sum_{j=1}^n \ell\paren{\tfwm\paren{\sequenceX{(j)}; \bV},y_j} - \E_{\bX,\by,\bW}\bigbrac{\dn \sumn \ell\paren{\tfwm,y_j}} }} \nonumber\\
     \le &~ 2\E_{\bX,\by,\bW,\boldsymbol{\xi}} \bigbrac{ \sup_{\bV \in \fnclass} \bigabs{ \dn \sum_{j=1}^n \xi_j \ell\paren{\tfwm\paren{\sequenceX{(j)}; \bV},y_j}}} \nonumber\\
     \lesssim &~ \Bvone \sqrt{\frac{\log (dM) \log \paren{nNM} }{n}},\label{equ::emp process E xy with E xyw mean}
    \end{align}
    where the last inequality is by Lemma~\ref{lem::emp process mean with E xyw}. Combining (\ref{equ::emp process E xy with E xyw concentration}) and (\ref{equ::emp process E xy with E xyw mean}) and taking $\epsilon = [\Bvtwo\log(6/\delta)/(M\wt{C})]^{1/2}$, we get 
    \begin{align*}
        &~ \sup_{\bV \in \fnclass} \bigabs{ \E_{\bX,\by} \bigbrac{ \dn \sumn \ell\paren{\tfwm,y_j}} - \E_{\bX,\by,\bW}\bigbrac{\dn \sumn \ell\paren{\tfwm,y_j}}}  \\ 
        \lesssim &~ \sqrt{\frac{\Bvtwo\log(1/\delta)}{M}} + \Bvone \sqrt{\frac{\log (dM)\log \paren{nNM} }{n}}
    \end{align*}
    with probability at least $1-\delta$. This proves Lemma \ref{lem::emp process E xy with E xyw}.
\end{proof}

\subsection{Proof of Theorem~\ref{thm::sample complexity}}
\label{sec:proof of sample complexity}
\begin{proof}
Given any target function $f_\star$, by Lemma~\ref{lem::rf and finite neuron}, with probability larger than $1-\delta/2$ over $\bW_{1:M}$, there exists $\wt{\bV}=\set{\wt{\bv}_m}_{m=1}^M$ such that 
\begin{equation}
    \norms{ f_\star - \tfwm(\sequenceX{};\wt{\bV})}_{L^2(\Px)} \lesssim
    \sqrt{\frac{(d^2+\log M) \rkhsbound \delta^{-1}}{M}}, \label{equ::sample complexity finite neuron}
\end{equation}
with 
\begin{equation}
    \sum_{m=1}^M \ltwob{\wt{\bv}_m} \lesssim \sqrt{\rkhsbound} + \sqrt{\frac{ \rkhsbound \delta^{-1}}{M}} \lesssim 2\sqrt{\rkhsbound} \text{ and } \sum_{m=1}^M \ltwob{\wt{\bv}_m}^2 \lesssim \frac{\rkhsbound \delta^{-1}}{M}. \label{equ::sample complexity fnclass}
\end{equation}
By our choice of $\Bvone$ and $\Bvtwo$, let $f^{\bW}_{\hat\bv,M} = \tfwm(\cdot;\hat\bV)$ denote the model trained by \eqref{equ::ERM fomula}, and let 
\begin{equation*}
    f^{\bW}_{\bv^*,M} = \argmin_{\bV \in \mathcal{V}_M} L_D\paren{\tfwm(\cdot;\bV)} ~~\text{and}~~ f^{\bW}_{\wt{\bv},M} = \tfwm(\sequenceX{};\wt{\bV}).
\end{equation*}
Then with probability at least $1-\delta$ over $\bW$, $\bX$, and $\by$, we have~\yub{fix}
\begin{align}
    &~ L_D\paren{f^{\bW}_{\hat\bv,M}} - L_D\paren{f^*} \\
    \le &~L_D\paren{f^{\bW}_{\hat\bv,M}} - \hat{L}_D\paren{f^{\bW}_{\hat\bv,M}} + \hat{L}_D\paren{f^{\bW}_{\hat\bv,M}} - \hat{L}_D\paren{f^{\bW}_{\bv^*,M}} + \hat{L}_D\paren{f^{\bW}_{\bv^*,M}} \nonumber \\
    &~ -  L_D\paren{f^{\bW}_{\bv^*,M}} + L_D\paren{f^{\bW}_{\bv^*,M}} - L_D\paren{f^{\bW}_{\wt{\bv},M}} + L_D\paren{f^{\bW}_{\wt{\bv},M}} -  L_D\paren{f_\star} \label{align::complexity line 1} \\
   \le &~ L_D\paren{f^{\bW}_{\hat\bv,M}} - \hat{L}_D\paren{f^{\bW}_{\hat\bv,M}} + \hat{L}_D\paren{f^{\bW}_{\bv^*,M}} - L_D\paren{f^{\bW}_{\bv^*,M}}  
    +  L_D\paren{f^{\bW}_{\wt{\bv},M}} -  L_D\paren{f^*} \label{align::complexity line 2} \\
    \le &~ 2 \sup_{f\in \fnclass}\bigabs{L_D\paren{f} - \hat{L}_D\paren{f}} + \norms{ f^{\bW}_{\wt{\bv},M} - f_\star }_{L^2(\Px)} \label{align::complexity line 3} \\
    \lesssim &~ \Bvone \sqrt{\frac{\log (dM) \log \paren{nNM}}{n}} + \sqrt{\log\paren{\frac{1}{\delta}}}\paren{\frac{\Bvone}{\sqrt{n}} + \sqrt{\frac{\Bvtwo}{M}}} + \sqrt{\frac{(d^2+\log M) \rkhsbound \delta^{-1}}{M}}  \label{align::complexity line 4} \\
    \lesssim &~ \sqrt{\frac{\rkhsbound}{n}} \paren{\sqrt{{\log (dM)\log(nNM)}} + \sqrt{{\log\paren{\delta^{-1}}}}} + \sqrt{\frac{(d^2+\log M) \rkhsbound \delta^{-1}}{M}} \label{align::complexity line 5} \\
    \lesssim &~ \sqrt{\frac{{\rkhsbound [\log (dM) \log (nNM)} +  \log (\delta^{-1})]}{{n}}} + \sqrt{\frac{(d^2 + \log M) \rkhsbound \delta^{-1}}{M}},
\end{align}
where from \eqref{align::complexity line 1} to \eqref{align::complexity line 2} we use the definition of $f^{\bW}_{\hat\bv,M}$ and $f^{\bW}_{\bv^*,M}$. From \eqref{align::complexity line 2} to \eqref{align::complexity line 3}, we bound~\yub{fix}~$L_D(f^{\bW}_{\hat\bv,M}) - \hat{L}_D ( f^{\bW}_{\hat\bv,M} )$ and $L_D ( f^{\bW}_{\bv^*,M}) - \hat{L}_D(f^{\bW}_{\bv^*,M})$ by $\sup_{f}| L_D\paren{f} - \hat{L}_D\paren{f} |$ and use the Lipschitzness of $\ell(f,y)$. From \eqref{align::complexity line 3} to \eqref{align::complexity line 4}, we use Proposition~\ref{prop::empirical process} and Lemma~\ref{lem::rf and finite neuron}. From \eqref{align::complexity line 4} to \eqref{align::complexity line 5}, we insert the value of $\Bvone$ and $\Bvtwo$ into the equation.
This proves Theorem~\ref{thm::sample complexity}. 
\end{proof}

\subsection{Proof of Examples in Section \ref{section:RKHS analysis} }\label{sec:proof_examples_RFA}

\subsubsection{Excess risk of \RFMLP}

Denote $\cF$ to be the set of all functions in the function class \eqref{eqn:target-function} and \eqref{eqn:target-expansion}, i.e., 
\begin{equation}\label{eqn:target_function_class_F}
\cF = \Big\{  f_\star(\sequenceX{})=\frac{1}{N}\sum_{i=1}^N\sum_{r,s \ge 0}^\infty \<\bx_0^{\otimes r}\otimes \bx_i^{\otimes s}, \mathbf{f}_{rs}\>:~\mathbf{f}_{rs} \in \R^{d^{r + s}} \text{ symmetric }, r,s\ge 0  \Big\}. 
\end{equation}
Consider the \RFMLP~model 
\begin{equation}
 \tflp(\sequenceX{}; \bv) = \sum_{m=1}^M \sigma\big(\<\bw_m,\operatorname{vec}(\sequenceX{})\> \big) \cdot v_m, ~~~~~ \{ \bw_m \}_{m \in [M]} \simiid \normal(\bzero,\id/(N+2)), \label{eqn::MLP model appendix}
\end{equation}
where $\operatorname{vec}(\sequenceX{}) = [\bx_0;\bx_1;\ldots;\bx_N;1] \in \R^{dN+d+1}$. For target functions that take forms in \eqref{eqn:target-function} and \eqref{eqn:target-expansion}, define $\mlpbound = \sum_{k=0}^\infty \wt{C}_{k} \sum_{r+s=k}\lfros{\mathbf{f}_{rs}}^2$ with $\wt{C}_{k} = k^{3.5} (N+2)^{2k}$. In case where $f_\star$ admits multiple representations of the form~\eqref{eqn:target-expansion}, $\mlpbound$ is the infimum of the right-hand-side over all such representations. 

Then we consider the empirical risk minimizer over the \RFMLP~model: 
\begin{equation}
\textstyle \hat\bv = \argmin_{\bv \in \mathcal{V}_M} \hat{L}_D(\tflp(\cdot;\bv)), ~~~~~~~~~ \hat{L}_D(f) = \dn \sum_{j=1}^n \ell(f(\sequenceX{(j)}),y_j), \label{eqn:RFMLP ERM fomula}
\end{equation}
where the constrained class $\mc{V}^{\textsc{MLP}}_M$ gives 
\begin{equation}
\textstyle \mc{V}^{\textsc{MLP}}_M = \left\{\bv = \set{v_m}_{m=1}^M : \,\, \sum_{m=1}^M \left|v_m\right| \le \Bvone , \sum_{m=1}^M v_m^2 \le \Bvtwo /M\right\}. \label{eqn:RFMLP ERM constraint}
\end{equation}

\begin{proposition}[The sample complexity of \RFMLP]
\label{prop::sample complexity MLP}
Let $f_\star = \arg\min_{f \in \cF}\populoss{f}$ be the population risk minimizer within the target function class \eqref{eqn:target_function_class_F}.
Assume $M > \delta^{-1}$ and $n > \log(dM)$. 
Take $\Bvone = C \sqrt{\mlpbound}$ and $\Bvtwo = C \mlpbound \delta^{-1}$ in \eqref{eqn:RFMLP ERM constraint}, with $C$ being a constant. Let $\hat{f}^{\textsc{MLP}}_{M} = \tflp(\cdot;\hat\bv)$ be empirical risk minimizer of model \RFMLP. Then for any joint distribution $\sfP$, with probability at least $1-\delta$ over $\sets{\bw_m}_{m\in[M]}$ sampled according to \eqref{eqn::MLP model appendix} and $\sets{(\sequenceX{(j)},y_j)}_{j\in[n]}\sim_{iid} \sfP$, the excess risk is bounded by %
\begin{equation}\label{equ::sample complexity MLP}
\begin{aligned}
    L_D(\hat{f}^{\textsc{MLP}}_{M}) -  L_D(f_\star) \le 
    \tO \paren{ \sqrt{\mlpbound} \bigg[\sqrt{\frac{1}{{n}}} + \sqrt{\frac{d^2\delta^{-1}}{M}} \bigg] }.
\end{aligned}
\end{equation}
\end{proposition}
\begin{proof}[Proof of Proposition \ref{prop::sample complexity MLP}]
The proof is basically the same as that of \RFA~model.
We only give a sketch of the proof.
Firstly, with a few modifications of the proof, we can show that Lemma \ref{lem::rf and finite neuron} also holds for \RFMLP, with $\rkhsbound$ replaced with $\mlpbound$. Lemma \ref{lem:RKHS_norm_target_function_Theorem1} is slightly different, we have the kernel expansion:
\[    K\paren{\sequenceX{},\sequenceX{\prime}}
    =\sum_{\ell\in\set{0,1}\cup \set{2k}_{k\ge
    1}}  \Big\langle \sqrt{c_\ell} (N+2)^{-\ell} \bigbrac{\operatorname{vec}(\sequenceX{})}^{\otimes \ell},\sqrt{c_\ell}  (N+2)^{-\ell} \bigbrac{\operatorname{vec}(\sequenceX{\prime})}^{\otimes \ell}  \Big\rangle.
\]
Therefore we can rewrite $f_\star$ as
\begin{equation}\label{eqn:expansion of fmlp}
f_\star = \sum_{\ell\in\set{0,1}\cup \set{2k}_{k\ge
    1}}  \Big\langle \frac{(N+2)^{\ell}}{N\sqrt{c_\ell}}\sum_{i=1}^N \sum_{r+s=\ell}\wt{\mathbf{f}}_{rs,i},\sqrt{c_\ell}  (N+2)^{-\ell} \bigbrac{\operatorname{vec}(\sequenceX{\prime})}^{\otimes \ell}  \Big\rangle,
\end{equation}
where $\langle\wt{\mathbf{f}}_{rs,i},[\operatorname{vec}(\sequenceX{\prime})]^{\otimes \ell}\rangle = \langle \mathbf{f}_{rs},\bx_0^{\otimes r}\otimes \bx_i^{\otimes s}\rangle$.
Thus, 
\[\|f_\star\|_{\mc{H}_K}^2 \le \sum_{k=0}^\infty \wt{C}_{k} \sum_{r+s=k}\lfros{\mathbf{f}_{rs}}^2~~~~\text{with}~~~~\wt{C}_{k} = k^{3.5} (N+2)^{2k}.\]
The right-hand-side gives the formulation of $\mlpbound$.  Then, a similar version of Proposition \ref{prop::empirical process} holds for \RFMLP~model $\tfwm$, i.e., with probability at least $1-\delta$ (w.r.t. $\bW$, $\by$, and $\bX$),
\begin{align}
 &~\sup_{\bV \in \mathcal{V}_M} \bigabs{ \frac{1}{n} \sum_{j=1}^n \ell\paren{\tfwm\paren{\sequenceX{(j)}; \bV},y_j} - \E_{(\sequenceX{}, y) \sim \sfP}\ell\paren{\tfwm\paren{\sequenceX{}; \bV},y}}\nonumber\\ 
 \lesssim &~ \Bvone \sqrt{\frac{(N+2)\log (dM) \log \paren{nNM}}{n}} + \sqrt{(N+2)\log\paren{1/\delta}}\paren{\frac{\Bvone}{\sqrt{n}} + \sqrt{\frac{\Bvtwo}{M}}}.
    \end{align}
    Then combining all of the above equations and following the proof in Section \ref{sec::proof of sample complexity}, we get \eqref{equ::sample complexity MLP}.
\end{proof}

\begin{remark}
The representation in \eqref{eqn:expansion of fmlp} is not unique. With a more careful choice of the representation of the target function $f_\star$, we can get a better bound for $\mlpbound$, which is given by 
\begin{equation}\label{eqn:better bound for mlp}
    \mlpbound = \sum_{k=0}^\infty \wt{C}_{k} \sum_{r+s=k}\lfros{\mathbf{f}_{rs}}^2~~~~\text{with}~~~~\wt{C}_{k} = k^{3.5} [(N+2)/2]^{2k}.
\end{equation}
\end{remark}
\subsubsection{Proofs of Examples}

\begin{proposition}[Restatement of Example \ref{exp:functions of x0}]\label{prop:example_functions_of_x0}
    For functions of $\bx_0$ of the form
\[
f_\star(\sequenceX{})=\sum_{k=0}^\infty \<\bx_0^{\otimes k}, \bA_k\>,~~ \bA_k\in\R^{d^k},
\]
we have that $B(f_\star) = \sum_{k=0}^\infty C_k\lfro{\bA_k}^2$.
Setting $M = \Theta( d^2 n )$, the excess risk bound gives $\tO ( \sqrt{\sum_{k=0}^\infty k^{4.5}4^k \| \bA_k \|_{\Fr}^2 / n} )$.  Moreover, consider $f_\star(\bx_{0:N})=(\bbeta^\top \bx_0)^p$. The above excess risk of \RFA~model and the \RFMLP{} model scales as
\[
\RFA: \tO \Big(\Poly(p) \sqrt{ 4^p \| \bbeta \|_2^{2p} /n } \,\Big),~~~~~~~{\RFMLP}:  \tO \Big( \Poly(p) \sqrt{((N+2))^p \| \bbeta\|_2^{2p}/n} \Big).
\]
\end{proposition}
\begin{proof}[Proof of Proposition \ref{prop:example_functions_of_x0}]
This follows by direct calculation.
\end{proof}

\begin{proposition}[Restatement of Example \ref{exp:functions of xi}]
\label{prop:restatement of exp function of xi}
For $f_\star = \frac{1}{N} \sum_{i=1}^N \psi(\<\bbeta,\bx_i\>)$ with $\psi(z) = z \arctan(z/\eta)$ with $\eta>2$ and $\ltwos{\bbeta} = 1$.  The excess risk bound of \RFA~model and the \RFMLP~model are 
\begin{align*}
\textstyle
\RFA{}: \tO\paren{\sqrt{ \sum_{k=1}^\infty k^{4.5}(2 /\eta)^{2k} / n }} = \tO(\sqrt{1/n}), \quad 
\RFMLP{}: \tO\paren{\sqrt{ \sum_{k=1}^\infty k^{4.5}[(N+2) / 2\eta]^{2k}  / n}}.
\end{align*}
\end{proposition}

\begin{proof}[Proof of Proposition \ref{prop:restatement of exp function of xi}]
Use the power series expansion of $\psi$, we have 
\[f_\star = \frac{1}{N} \sum_{i=1}^N \sum_{k=1}^\infty (-1)^k \frac{\<\bbeta,\bx_i\>^{2k}}{(2k-1)\eta^{2k-1}}.\]
Plug it into the formula of $\rkhsbound$ \eqref{eqn:bfstar} and $\mlpbound$ \eqref{eqn:better bound for mlp}, we get
\begin{align*}
\rkhsbound = \sum_{k=1}^\infty (2k)^{4.5} 4^{2k} \| \bbeta/\eta \|^{4k}_2 = \Theta\Big( \sum_{k=1}^\infty k^{4.5} (2/\eta)^{2k}\Big),
\end{align*}
and $\mlpbound = \Theta( \sum_{k=1}^\infty k^{4.5} [(N+2)/2\eta]^{2k})$.
Therefore, by Theorem \ref{thm::sample complexity} and Proposition \ref{prop::sample complexity MLP}, we get their excess risk. This proves Proposition \ref{prop:restatement of exp function of xi}.
\end{proof}

\begin{proposition}[Restatement of Example \ref{exp:general correlation functions}]
\label{prop:restatement of general correlation functions}
For
$f_{1, \star} = \frac{1}{N} \sum_{i=1}^N \<\bx_0,\bx_i\>^p$, the excess risk bound of \RFA{} (by Theorem~\ref{thm::sample complexity}) and \RFMLP{} scale as 
\[
\textstyle
\RFA{}: \tO\paren{\Poly(p) \sqrt{ (4d)^p / n} },~~~~~~~ \RFMLP{}: \tO\paren{\Poly(p) \sqrt{ [(N+2)d]^p / n} }.
\]
For $f_{2,\star} =  \frac{1}{N} \sum_{i=1}^N \cos(\<\bx_0,\bx_i\>)\<\bx_i^{\otimes p},\bG\>$ with $\lfros{\bG}=1$. Then the excess risk bound of \RFA{} and \RFMLP{} scale as 
\[\textstyle
\RFA{}: \tO\paren{\Poly(p d) \sqrt{ e^{4\sqrt{d}} 4^p / n}}, ~~~~~ \RFMLP{}: \tO\paren{\Poly(p N d) \sqrt{ e^{2(N+2)d} (N+2)^p /n}}.
\]
\end{proposition}
\begin{proof}[Proof of Proposition \ref{prop:restatement of general correlation functions}]
For $f_{1,\star}$, direct calculation gives the value of $B(f_{1,\star})$ and the excess risk follows. 
For $f_{2,\star}$, using the Taylor expansion of $\cos(z)$, we get
\[
f_{2,\star} = \frac{1}{N} \sum_{i=1}^N  \sum_{k=0}^\infty (-1)^k \frac{\<\bx_0,\bx_i\>^{2k}}{(2k)!} \<\bx_i^{\otimes p},\bG\> = \frac{1}{N} \sum_{i=1}^N \sum_{k=0}^\infty \< (\bx_0 \bx_i^\top)^{\otimes 2k} \otimes \bx_i^{\otimes p} ,\frac{(-1)^k}{(2k)!} \bI_{d}^{\otimes 2k}\otimes\bG\>.
\]
Plug it into the formula of $\rkhsbound$ \eqref{eqn:bfstar} and $\mlpbound$ \eqref{eqn:better bound for mlp}, we get
\begin{align*}
    B(f_{2,\star}) = 4^p \Poly(p) \sum_{k=0}^\infty (2k)^{4.5} 4^{2k}\frac{d^{2k}}{((2k)!)^2} = 4^p \Poly(pd) \Theta\Big(\sum_{k=0}^\infty \frac{(4d)^{2k}}{((2k)!)^2}\Big).
\end{align*}
Note that for any $z>0$,
\begin{equation}\label{eqn:upper bound expansion arctan}
    \sum_{k=0}^\infty \frac{(z)^{2k}}{((2k)!)^2} \le  \bigbrac{\sum_{k=0}^\infty \frac{(\sqrt{z})^{2k}}{(2k)!}}^2 = \Theta\Big(e^{2\sqrt{z}}\Big).
\end{equation}
Plug $z=4d$ into \eqref{eqn:upper bound expansion arctan} for $B(f_{2,\star})$ gives the excess risk for \RFA~model. As for $B_{\rm MLP}(f_{2,\star})$, we have
\[
B_{\rm MLP}(f_{2,\star}) = \Poly(p) \sum_{k=0}^\infty (N+2)^p (4k)^{4.5} (N+2)^{4k}\frac{d^{2k}}{((2k)!)^2} = 4^p \Poly(pdN) \Theta\Big(\sum_{k=0}^\infty \frac{((N+2)\sqrt{d})^{4k}}{((2k)!)^2} \Big).
\]
Using similar argument, we get that $B_{\rm MLP}(f_{2,\star})$ is bounded by $\Theta(\Poly(pdN)4^p \exp(2{(N+2)\sqrt{d}}))$. Using Theorem \ref{thm::sample complexity} and Proposition \ref{prop::sample complexity MLP}, we can get the excess risk bound. This proves Proposition \ref{prop:restatement of general correlation functions}.
\end{proof}

\section{Proofs for Section \ref{section:bias attention}}
\label{sec::bias sample complexity proof}

We consider the empirical risk minimizer over the \BRFA~model \eqref{eqn:RF_attention_bias}, 
\begin{equation}
\textstyle \hat\bV = \argmin_{\bV \in \mathcal{V}_M} \hat{L}_D(\tfwmb(\cdot;\bV)), ~~~~~~~~~ \hat{L}_D(f) = \dn \sum_{j=1}^n \ell(f(\sequenceX{(j)}),y_j), \label{eqn:BRFA ERM fomula}
\end{equation}
where the constrained class $\mathcal{V}_M$ gives \eqref{equ::ERM constraint}, copied here for reader's convenience, 
\begin{equation}
\textstyle \cV_M = \left\{\bV = \set{\bv_m}_{m=1}^M : \,\, \sum_{m=1}^M \ltwo{\bv_m} \le \Bvone , \sum_{m=1}^M \ltwo{\bv_m}^2 \le \Bvtwo /M\right\}. \label{eqn:BRFA ERM constraint}
\end{equation}
Denote $\cG$ to be the set of all functions in the function class \eqref{eqn:bias rf hard}, i.e., 
\begin{equation}\label{eqn:target_function_class_G}
\cG = \Big\{  \fhard = \frac{1}{N} \sum_{i=1}^N F(\langle \bx_0,\bx_i \rangle) G(\bx_0,\bx_i):
F(t) = \sum_{k=0}^\infty a_k t^k, G = \<\tbx_i^{\otimes 3} \otimes \tbx_0^{\otimes 2}, \bA_\star\> \Big\}.
\end{equation}
We restate Theorem \ref{thm::bias sample complexity} in Theorem \ref{thm::sample complexity biased appendix} with detailed assumptions.

\begin{theorem}[Restatement of Theorem \ref{thm::bias sample complexity}]
\label{thm::sample complexity biased appendix}
Assume $M > \delta^{-1}$, $n > \log(dM)$ and let $L\in \Z_{\ge 1}$. Let $g_\star = \arg\min_{g \in \cG}\populoss{g}$ be the population risk minimizer within the target function class $\cG$ \eqref{eqn:target_function_class_G}.
Take $\Bvone = C \sqrt{B\paren{\fhard,L}}$ and $\Bvtwo = C B(\fhard,L) \delta^{-1}$ in \eqref{eqn:BRFA ERM constraint}, with $C$ being a constant. Let $\hat{f}^{\bW,\bW_0}_{M} = \tfwmb(\cdot;\hat\bV)$ be empirical risk minimizer given by \eqref{eqn:BRFA ERM fomula}. Then for any joint distribution $\sfP$, with probability at least $1-\delta$ over $\sets{\bW_m}_{m\in[M]}$ sampled according to \eqref{eqn:w-distribution} and $\sets{(\sequenceX{(j)},y_j)}_{j\in[n]}\sim_{iid} \sfP$, the excess risk is bounded by %
\begin{align*}
&~\populoss{\hat{f}^{\bW,\bW_0}_{M}} -  L_D(\fhard) \nonumber\\
\lesssim&~ \inf_{L}\Bigg\{  \sqrt{B(\fhard, L)} \Big[ \sqrt{\frac{{\log (dM) \log (nNM)} +  \log (\delta^{-1})}{{n}}} + \sqrt{\frac{(d^2 + \log M) \delta^{-1}}{M}} \Big] + \eps_L \linf{\fhard} \Bigg\},
\end{align*}
where $\epsilon_L = 1/[2^{L+1}(L+1)!]$ and
\begin{align}
\textstyle
    B(\fhard, L) = \lfro{\bA_\star}^2 \cdot 
( \sum_{k=0}^\infty | a_k | \cdot C_k )^2, ~~~~\text{\rm with } C_k = (2L+k)^{(k+3)/2} 8^{L+ k/2}. 
\label{eqn:bfstar for biased appendix}
\end{align}
\end{theorem}

\subsection{Auxiliary results for the proof of Theorem \ref{thm::bias sample complexity}}\label{sec:pre bias sample complexity}

In this section, we give some auxiliary results used in the proof of Theorem \ref{thm::bias sample complexity}. The proof will be given in Section \ref{sec:prf bias sample complexity}. We define the biased transformer with infinite width (informally corresponding to Eq. \eqref{eqn:RF_attention_bias} with $M\to\infty$), given by
\begin{align}
    \fwz_{\bv}(\sequenceX{}) &= \frac{1}{N}\sum_{i = 1}^N \E_{\bW}\brac{ \sigma\paren{\<\bW +\bW_0, \tbX\>} \<\bv(\bW), \tbx_i\>} \label{eqn:infinite_bias_RF}\\
    &= \frac{1}{N}\sum_{i = 1}^N \E_{\bW}\brac{ \sigma\paren{\<\bW, \tbX\>+h_i} \<\bv(\bW), \tbx_i\>},  \nonumber
\end{align}
where $h_i \defeq \langle \bW_0,\tbX \rangle$ and $\bW$ is sampled according to \eqref{eqn:w-distribution}. Here we set $\bW_0 = [ \bId, \bzero_{d \times 1}; \bzero_{1 \times d}, 0]$. Then $h_i=\<\bx_0,\bx_i\>$.

We consider a class of target functions $\feasy: \mathcal{X} \to \mathbb{R}$ that takes form %
\begin{equation}
    \feasy(\sequenceX{}) = \frac{1}{N} \sum_{i=1}^N \sum_{\ell = 0}^\infty \< u_\ell\paren{\tbx_0,\tbx_i}, \bD_\ell \>, \label{equ::rf biased target function}
\end{equation}
for some coefficients $\sets{\bD_\ell \in \mathbb{R}^{(d+1)^{2\ell + 1}}}_{\ell \ge 0}$ and $$u_\ell\paren{\tbx_0, \tbx_i} = \left\{\begin{matrix}
    \brac{h_i \Phi\paren{h_i } +  \phi\paren{h_i }} \tbx_i & \paren{\ell=0}, \\
     \Phi\paren{h_i } \tbX \otimes \tbx_i & \paren{\ell=1}, \\
     \phi\paren{h_i } \He_{\ell-2}\paren{h_i } (\tbX)^{\otimes \ell} \otimes \tbx_i & \paren{\ell\geq 2}, \\
\end{matrix} \right.$$
where $\phi$ and $\Phi$ are the PDF and CDF of the standard Gaussian random variable, respectively. Lemma \ref{lemma:bias expansion} below provides a counterpart of Lemma \ref{lem:RKHS_norm_target_function_Theorem1} for the \BRFA~model and the target function \eqref{equ::rf biased target function}.
\begin{lemma}\label{lemma:bias expansion}
    Any function $\feasy$ of form~\eqref{equ::rf biased target function} can be expressed exactly as an infinite-head random feature attention model \eqref{eqn:infinite_bias_RF} %
\begin{equation}\label{equ::rf bias v}
            \feasy(\sequenceX{}) = \fwz_\bv(\sequenceX{}) = \frac{1}{N}\sum_{i = 1}^N \E_{\bW}\brac{ \sigma\paren{\<\bW + \bW_0, \tbX\>} \<\bv(\bW), \tbx_i\>},
    \end{equation}
    where the coefficients $\bv(\cdot)$ satisfy 
    \begin{equation*}
        \E_\bW\brac{\ltwo{\bv\paren{\bW}}^2} \le \sum_{\ell \ge 0} 4^\ell\ell! \lfro{\bD_\ell}^2.
    \end{equation*}
\end{lemma}
Given Lemma \ref{lemma:bias expansion}, we can get a counterpart of Theorem \ref{thm:finite-width} for the \BRFA~model as Proposition \ref{prop:finite width biased} below. 
\begin{proposition}[Counterpart of Theorem \ref{thm:finite-width} for {\BRFA}]
\label{prop:finite width biased}
Suppose function $\feasy:\cX\to\R$ takes form~\eqref{equ::rf biased target function}. Then for any input distribution $\Px$ on $\cX$, with probability at least $1-\delta$ (over $\sets{\bW_m}_{m\in[M]}$ sampled from~\eqref{eqn:w-distribution}), there exists an $M$-head \BRFA~model (\ref{eqn:RF_attention_bias}) with coefficients $\bV = \sets{\bv_m}_{m\in[M]}\subseteq \R^{d+1}$ that approximates $\feasy$ in $L^2(\Px)$ up to error
\begin{align}
    \E_{\bx_{0:N}\sim \Px}\brac{ \paren{ \feasy(\bx_{0:N}) -\tfwm(\bx_{0:N}; \bV) }^2 } \le \cO\bigg(\frac{ (d^2 + \log M) B(\feasy) \delta^{-1}}{ M} \bigg). \label{eqn:approx Id}
\end{align}
In addition, the norms of the weight of this random-feature attention model are bounded as
\begin{align}
    \sum_{m=1}^M \ltwo{\bv_m} \le \cO \bigg( \sqrt{B(\feasy)} + \sqrt{\frac{ B(\feasy) \delta^{-1}}{M}} \bigg),~~~~~~ \sum_{m=1}^M \ltwo{\bv_m}^2 \le \cO \bigg( \frac{ B(\feasy) \delta^{-1}}{M} \bigg). \label{eqn:norm bound of v Id}
\end{align}
Here $B(\feasy)$ is defined alternatively as 
\begin{align}
\label{eqn:bfstar-2}
\textstyle B(\feasy) = \sum_{k=0}^\infty C_{k} \lfros{\bD_k}^2, ~~~~\text{with}~~~~~ C_k = k!4^k\vee 1.
\end{align}
\end{proposition}

Furthermore, we could approximate the target function in the form ~\eqref{eqn:bias rf hard}  to any precision by a function that takes form  ~\eqref{equ::rf bias v} , which is discussed in Lemma \ref{lemma:bias approx}.
\begin{lemma}\label{lemma:bias approx}
    For any target function $\fhard$ in the form of ~\eqref{eqn:bias rf hard}, and for any precision $\eps_\ell \defeq \frac{1}{2^{\ell+1}(\ell+1)!}$, there exists a function $\fwz_\bv$ in the form of ~\eqref{equ::rf bias v} such that
    \begin{equation*}
        \linf{\fhard-\fwz_\bv} \le \linf{\fhard}\eps_\ell, 
    \end{equation*}
    and
    \begin{equation*}
        \E_\bW\brac{\ltwo{\bv\paren{\bW}}^2} \le \paren{\sum_{p=0}^\infty a_p  8^{\ell+\frac{p+3}{2}} (2\ell+p)^{\frac{p+3}{2}}}^2\lfro{\bA_\star}^2,
    \end{equation*}
    where the tensor $\bA_\star \in \R^{d^5}$ parameterizes $\fhard$. 
\end{lemma}

\subsubsection{Proof of auxiliary lemmas}\label{prf:bias lemmas}
\begin{proof}[Proof of Lemma \ref{lemma:bias expansion}]~

The function \eqref{equ::rf bias v} can be interpreted as a linear function of the feature map
 $$\Psi (\sequenceX{})=\frac{1}{N}\sum_{i =1}^N\barsig\paren{\langle \bW, \tbX \rangle +h_i} \tbx_i. $$
So the kernel w.r.t. the feature map takes the form that
\begin{align*}
&~K\paren{\sequenceX{},\sequenceX{\prime}} \\
    = &~\E_{\bW}\brac{\langle \Psi (\sequenceX{}), \Psi (\sequenceX{\prime}) \rangle} \\
    =&~\frac{1}{N^2}\sum_{1\leq i,j \leq N} \E_{\bW}\brac{\barsig\paren{\langle \bW,\tbX\rangle+h_i}\barsig\paren{\langle\bW,\tbx_0^\prime\tbx_{j}^{\prime\top}\rangle+h_j^\prime} }\langle \tbx_i, \tbx_j'\rangle,
\end{align*}
where $h_i=\<\bx_0,\bx_i\>$ and $h_j'=\<\bx'_0,\bx'_j\>$. Similar to the proof of Lemma \ref{lem:RKHS_norm_target_function_Theorem1},  we also consider a single component of the equation above first, which has the form of
\begin{equation}\label{eq:Gaussian plus identity single comp}
    \E_{\bW}\brac{\barsig\paren{\langle \bW,\tbX\rangle +h_i}\barsig\paren{\langle\bW,\tbx_0^\prime\tbx_{j}^{\prime\top}\rangle+h_j^\prime} }\langle \tbx_i, \tbx_j'\rangle.
\end{equation}
We expand $\barsig\paren{\langle \bW,\tbX \rangle + h_i}$ by Hermite polynomials in the space $L^2(\R, \phi)$ using Lemma \ref{clm::relu with hermite},
\begin{align*}
    &~\E_{\bW}\brac{\barsig\paren{\langle \bW,\tbX \rangle +  h_i}\He_\ell\paren{{\< \bW,\tbX \>}}} \\
    = &~ \E_{z\sim\normal(0,1)}\brac{ \barsig\paren{z +  h_i }\He_\ell(z)} \\
    = &~ \left\{
    \begin{matrix}
    h_i \Phi\paren{h_i } +  \phi(h_i ) & \paren{\ell=0}, \\
      \Phi\paren{h_i } & \paren{\ell=1}, \\
    (-1)^\ell  \phi\paren{h_i } \He_{\ell-2}\paren{h_i } & \paren{\ell\geq 2}. \\
    \end{matrix}
    \right.
\end{align*}
Therefore, we have
\begin{align*}
    \barsig\paren{\langle \bW,\tbX \rangle + h_i} = &~h_i \Phi\paren{h_i } +  \phi\paren{h_i } +   \Phi\paren{h_i } \He_1\paren{\<\bW,\tbX\>} \\
     &~ +  \sum_{\ell=2}^\infty \frac{(-1)^\ell}{\ell!}  \phi\paren{h_i } \He_{\ell-2}\paren{h_i } \He_\ell\paren{\<\bW,\tbX\>}. 
\end{align*}

Using Lemma \ref{lemma:Hermite inner product}, we obtain the expansion: 
\begin{align*}
&~ \E_{\bW}\brac{\barsig\paren{\langle\bW,\tbX\rangle+h_i}\barsig\paren{\langle\bW,\tbx_0^\prime\tbx_j^{\prime\top}\rangle+h'_j} }\langle \tbx_i, \tbx_j'\rangle \\
     = &~  \brac{h_i \Phi\paren{h_i } +  \phi\paren{h_i }}\brac{h'_j \Phi\paren{h'_j} + \phi\paren{h'_j}} \langle \tbx_i, \tbx_j'\rangle \\
    &~ +   2^{-2} \Phi\paren{h_i } \Phi\paren{h'_j} \<\tbX,\tbx_0^\prime\tbx_j^{\prime\top}\> \langle \tbx_i, \tbx_j'\rangle
    \\ 
    &~ + \sum_{\ell=2}^\infty\left\{ \frac{1}{\ell!}  2^{-2\ell} \phi\paren{h_i} \phi\paren{h'_j} \He_{\ell-2}\paren{h_i} \He_{\ell-2}\paren{h'_j}  \<\tbX,\tbx_0^\prime\tbx_j^{\prime\top}\>^{\ell} \right\} \langle \tbx_i, \tbx_j'\rangle\\
    = &~  \sum_{\ell=0}^\infty \< \varphi_\ell\paren{\tbx_0,\tbx_i} , \varphi_\ell(\tbx'_0,\tbx'_j)\>,
\end{align*}
where  \begin{equation*}
\varphi_\ell\paren{\tbx_0,\tbx_i}=\sqrt{\frac{1}{\ell!}} 2^{-\ell}u_\ell\paren{\tbx_0,\tbx_i} = \left\{\begin{matrix}
    \brac{h_i \Phi\paren{h_i } +  \phi\paren{h_i }} \tbx_i & \paren{\ell=0}, \\
      2^{-1} \Phi\paren{h_i } \tbX \otimes \tbx_i & \paren{\ell=1}, \\
   \sqrt{\frac{1}{\ell!}} 2^{-\ell} \phi\paren{h_i } \He_{\ell-2}\paren{h_i } (\tbX)^{\otimes \ell} \otimes \tbx_i & \paren{\ell\geq 2}. \\
\end{matrix} \right.
\end{equation*}
Then by taking summation with respect to $i$ and $j$,  we derive the expansion of the kernel: 
\begin{equation*}
K\paren{\sequenceX{},\sequenceX{\prime}}=\sum_{\ell=0}^\infty \Big\langle \frac{1}{N}\sum_{i=1}^N\varphi_\ell\paren{\bx_0,\bx_i} , \frac{1}{N}\sum_{j=1}^N\varphi_\ell(\bx_0^\prime,\bx_j^\prime)\Big\rangle.
\end{equation*}
Then for the target function that takes the form in \eqref{equ::rf biased target function}, we have the RKHS norm of $\feasy$ bounded by: %
\[\|\feasy\|_{\mc{H}_{K}}^2 \le \sum_{\ell=0}^\infty \Big(2^\ell\sqrt{\ell!}\Big)^2 \lfros{\bD_\ell}^2 = \sum_{\ell=0}^\infty 4^\ell \ell! \lfro{\bD_\ell}^2 \]
by the feature equivalence property \eqref{eqn:feature map equi}. Thus, there exists coefficients $\bv(\bW)$ such that
    \begin{equation*}
            \fwz_\bv(\sequenceX{}) = \frac{1}{N}\sum_{i = 1}^N \E_{\bW}\brac{ \sigma\paren{\<\bW + \bW_0, \tbX\>} \<\bv(\bW), \tbx_i\>}
    \end{equation*}
    and
    \begin{equation*}
        \E_\bW\brac{\ltwo{\bv\paren{\bW}}^2} \le \sum_{\ell \ge 0} 4^\ell \ell! \lfro{\bD_\ell}^2,
    \end{equation*}
which proves Lemma \ref{lemma:bias expansion}.
\end{proof}

\begin{proof}[Proof of Proposition~\ref{prop:finite width biased}]~

Note that Lemma~\ref{lem::rf and finite neuron} is also applicable to the model \BRFA. Combining it with Lemma~\ref{lemma:bias expansion} proves Proposition~\ref{prop:finite width biased}.
\end{proof}

\begin{proof}[Proof of Lemma~\ref{lemma:bias approx}]~

    Firstly, we consider approximating $x^p$ using the function class $\sets{\phi(x)\He_\ell(x)}_{n\ge 0}$, which is equivalent to approximating $x^p\phi^{-1}(x)$ using Hermite polynomials. We take $p=0$ first and compute the Hermite expansion of the  $2\ell$-th-order term in the Taylor expansion of $\phi^{-1}(x)=e^{x^2/2}$, i.e. $\psi_\ell(x)=\sum_{n=0}^\ell\frac{x^{2n}}{2^n n!}$. Using Lemma \ref{clm:Inverse explicit expression}, we have that:
\begin{align*}
    \psi_\ell(x) =&~\sum_{n=0}^\ell\frac{x^{2n}}{2^n n!} = \sum_{n=0}^\ell \frac{(2n)!}{2^n n!}\sum_{m=0}^n \frac{\He_{2n-2m}(x)}{2^m m! (2n-2m)!}\\
    =&~\sum_{k=0}^\ell \He_{2k}(x) \sum_{m=0}^{\ell -k} \frac{(2m+2k)!}{2^{m+k}(m+k)!}\cdot \frac{1}{ 2^m m! (2k)!} \\
    \eqdef&~ \sum_{k=0}^\ell c_{\ell,k} \He_{2k}(x),
\end{align*}
where $c_{\ell,k} =\frac{2^k}{ (2k)!}\sum_{m=0}^{\ell -k} \frac{(2m+2k)!}{2^{2m+2k}(m+k)!m!}$.
Similarly, for any $p\ge 0$, define $p_0=\lfloor \frac{p}{2} \rfloor$ with $r = p-2p_0$. We have that 
\begin{align*}
    \psi_\ell(x)x^p = &~\sum_{n=0}^\ell\frac{x^{2n+2p_0+r}}{2^n n!} = \sum_{n=0}^\ell \frac{(2n+2p_0+r)!}{2^n n!}\sum_{m=0}^{n+p_0} \frac{\He_{2n+2p_0+r-2m}(x)}{2^m m! (2n+2p_0+r-2m)!}\\
    = &~ \sum_{k=0}^{\ell+p_0} \He_{2k+r}(x) \sum_{m=\max(p_0-k,0)}^{\ell +p_0 -k} \frac{(2m+2k+r)!}{2^{m+k-p_0}(m+k-p_0)!}\cdot \frac{1}{ 2^m m! (2k+r)!} \\
    \eqdef &~ \sum_{k=0}^{\ell+p_0} c_{\ell,k,p_0} \He_{2k+r}(x),
\end{align*}
where $c_{\ell,k,p_0} = \sum_{m=\max(p_0-k,0)}^{\ell +p_0 -k} \frac{(2m+2k+r)!}{2^{m+k-p_0}(m+k-p_0)!}\cdot \frac{1}{ 2^m m! (2k+r)!}$.
Then we can bound $c_{\ell,k,p_0}$ as follows:
\begin{align*}
     c_{\ell,k,p_0} & =\sum_{m=\max(p_0-k,0)}^{\ell +p_0 -k} \frac{(2m+2k+r)!}{2^{m+k-p_0}(m+k-p_0)!}\cdot \frac{1}{ 2^m m! (2k+r)!} \\
     & \leq \sum_{m=0}^{\ell +p_0 -k} 2^{p_0}\frac{(m+k)!}{(m+k-p_0)!} \frac{(2m+2k+r)!}{2^{m+k}(m+k)!}\cdot \frac{1}{ 2^m m! (2k+r)!} \\
     & \leq 2^{p_0}(\ell+p_0)^{p_0}(2l+p)^r \sum_{m=0}^{\ell +p_0 -k} \frac{(2m+2k)!}{2^{m+k}(m+k)!}\cdot \frac{1}{ 2^m m! (2k)!} \leq (2\ell+p)^{p_0+r} c_{\ell+p_0,k}.
\end{align*}
Next, we give an upper bound of $c_{\ell,k}$:
\begin{align*}
    c_{\ell,k} & =\frac{2^k}{ (2k)!}\sum_{m=0}^{\ell -k} \frac{(2m+2k)!}{2^{2m+2k}(m+k)!m!} \\
    &= \frac{2^k}{ (2k)!}\sum_{m=0}^{\ell -k} \frac{(2m+2k)!}{2^{2m+2k}(m+k)!(m+k)!} \cdot \frac{(m+k)!}{m!} \\
    & \leq \frac{2^k}{ (2k)!}\sum_{m=0}^{\ell -k} \sqrt{\frac{1}{m+k}} \cdot \frac{(m+k)!}{m!} \\
    & = \frac{2^k}{ (2k)!}\sum_{m=0}^{\ell -k} \sqrt{m+k} \cdot \frac{(m+k-1)!}{m!} \\
    & \leq \frac{2^k}{ (2k)!}\sqrt{\ell} \sum_{m=0}^{\ell -k} \frac{(m+k-1)!}{m!} \\
    & =    \frac{2^k}{ (2k)!}\sqrt{\ell} \frac{\ell !}{k(\ell-k)!}.
\end{align*}
Here we use the inequality that
\[
\sqrt{\frac{2}{\pi (2n+1)}} \le \frac{(2n)!}{2^{2n}n!n!} \le \sqrt{\frac{1}{2n}}, \forall n\ge 1.
\]
and we will use it again in the following proof.
Since the function 
$$\frac{1}{N} \phi(\<\bx_0,\bx_i\>)\psi_\ell(\<\bx_0,\bx_i\>)\sum_{i=1}^N \<\bx_0,\bx_i\>^p \<\tbx_i^{\otimes 3} \otimes \tbx_0^{\otimes 2}, \bA_\star \> $$
can be written as
\begin{equation*}
    \frac{1}{N} \sum_{i=1}^N \sum_{k=0}^{\ell+p_0}\< u_{2k+r+2}\paren{\tbx_0,\tbx_i},  c_{\ell,k,p_0}\bA_\star\otimes \be^{\otimes (4k+2r)} \>,
\end{equation*}
where $\be=\diag(0,0,...,1)$. By Lemma~\ref{lemma:bias expansion}, there exists $\bv_{p,\ell}(\bW)$ s.t. 
\begin{equation*}
    \fwz_{\bv_{p,\ell}}(\bx_{0:N})= \phi(\<\bx_0,\bx_i\>)\psi_\ell(\<\bx_0,\bx_i\>)\frac{1}{N} \sum_{i=1}^N \<\bx_0,\bx_i\>^p\<\tbx_i^{\otimes 3} \otimes \tbx_0^{\otimes 2}, \bA_\star \>
\end{equation*}
with
\begin{align*}
    \E_\bW\brac{\ltwo{\bv_{p,\ell}\paren{\bW}}^2} &\le \lfro{\bA_\star}^2\sum_{k=0}^{\ell+p_0} 4^{2k+r+2}(2k+r+2)!  c_{\ell,k,p_0}^2.
\end{align*}
Notice that
\begin{align*}
    &~\sum_{k=0}^{\ell+p_0} 4^{2k+r+2}(2k+r+2)! c_{\ell,k,p_0}^2 \\
    \le&~ \sum_{k=0}^{\ell+p_0}  4^{2k+r+2}(2k+r+2)!\brac{(2\ell+p)^{p_0+r} c_{\ell+p_0,k}}^2 \\
    \le &~(2\ell+p)^{2p_0+2r} \sum_{k=0}^{\ell+p_0}4^{2k+r+2} (2k+r+2)! \brac{\frac{2^k}{ (2k)!}\sqrt{\ell+p_0} \frac{(\ell+p_0) !}{k(\ell+p_0-k)!}}^2\\
    \le &~ 4(2\ell+p)^{2p_0+2r}(2\ell+2p_0+r+2) \sum_{k=0}^{\ell_p}4^{2k+2} (2k+2)! \brac{\frac{2^k}{ (2k)!}\sqrt{\ell_p} \frac{(\ell_p) !}{k(\ell_p)!}}^2,
\end{align*}
where $\ell_{p}=\ell+p_0$, and by the inequality that 
\begin{align*}
    &~\sum_{k=0}^{\ell_p}4^{2k+2} (2k+2)! \brac{\frac{2^k}{ (2k)!}\sqrt{\ell_p} \frac{(\ell_p) !}{k(\ell_p)!}}^2 \\
    = &~ \sum_{k=0}^{\ell_p} 2^{4k + 4}\ell_p \frac{(2k+2)(2k+1)}{k^2} \frac{2^{2k}}{ (2k)!}\brac{\frac{\ell_p !}{(\ell_p-k)!}}^2 \\
    \le &~ \sum_{k=0}^{\ell_p} 2^{4k + 4}\ell_p \frac{(2k+2)(2k+1)}{k^2} \sqrt{\frac{\pi}{2} (2k+1)} \brac{\frac{\ell_p !}{k!(\ell_p-k)!}}^2 \\
     \le &~ 8 \times  2^{4\ell_p + 4}\ell_p^{\frac{3}{2}} \sum_{k=0}^\ell \brac{\frac{\ell_p !}{k!(\ell_p-k)!}}^2 \\
     = &~ 4\sqrt{2} \times 2^{4\ell_p + 4}\ell_p^{\frac{3}{2}}  \binom{2\ell_p}{\ell_p} \\
     \le &~4\sqrt{2} \times 
 2^{4}\ell_p \cdot (64)^{\ell_p},
\end{align*}
we obtain an upper bound:
\begin{align*}   \E_\bW\brac{\ltwo{\bv_{p,\ell}\paren{\bW}}^2} &\le (2\ell+p)^{2p_0+2r}\lfro{\bA_\star}^2 \sum_{k=0}^{\ell+p_0} 4^{2k+r+2}(2k+r+2)!  c_{\ell+p_0,k}^2\\ 
&\le \lfro{\bA_\star}^2 \times 4(2\ell+p)^{2p_0+2r}(2\ell+2p_0+r+2) \times 4\sqrt{2} \times 
 2^{4}\ell_p \cdot (64)^{\ell_p}\\
 &\le \lfro{\bA_\star}^2(2\ell+p)^{p+r+1}8^{2\ell+2p_0+3}.
\end{align*} 
Finally, we consider the target function \eqref{eqn:bias rf hard}
\begin{equation*}
    \fhard(\bx_{0 : N}) = \frac{1}{N} \sum_{i = 1}^N F(\<\bx_0, \bx_i\>) G(\bx_0, \bx_i), \quad F(t) = \sum_{p=0}^\infty a_p t^p, \quad G(\bx_0, \bx_i) = \<\tbx_i^{\otimes 3} \otimes \tbx_0^{\otimes 2}, \bA_\star\>. 
\end{equation*}
By approximating $\frac{1}{N} \sum_{i = 1}^N a_p\<\bx_0, \bx_i\>^p G(\bx_0, \bx_i)$ separately and adding the approximation functions together, we obtain a $\bv_{\ell}$ s.t.
$\fwz_{\bv_{\ell}}(\bx_{0 : N})= 
\phi(\<\bx_0,\bx_i\>)\psi_\ell(\<\bx_0,\bx_i\>)\fhard(\bx_{0 : N})$ and
\[
\E_\bW\brac{\ltwo{\bv_{p,\ell}\paren{\bW}}^2} \le  \paren{\sum_{p=0}^\infty a_p  8^{\ell+p/2+3/2} (2\ell+p)^{\frac{p+3}{2}}}^2\lfro{\bA_\star}^2=8^3\biasbound
\]
and $\linf{\fhard-\fwz_{\bv_{\ell}}}=\linf{(\phi\psi_\ell-1)\fhard}\leq\linf{\phi\psi_\ell-1}\linf{\fhard}\leq \frac{e}{2^{\ell+1}(\ell+1)!}\linf{\fhard}=\linf{\fhard}\epsilon_\ell$.
This finishes the proof of Lemma \ref{lemma:bias approx}.
\end{proof}

\subsection{Proof of Theorem \ref{thm::sample complexity biased appendix}}\label{sec:prf bias sample complexity}
\begin{proof}
For any $L>0$, using Lemma \ref{lemma:bias approx}, we can find a function $\fwz_{\bv_L}$ such that $\linf{\fhard-\fwz_{\bv_L}} \le \linf{\fhard} \eps_L$ with $B(\fwz_{\bv_L}) \le 8^3B(\fhard,L)$. Follow the same manner as the proof of Theorem \ref{thm::sample complexity}. We can get that 
\begin{align*}
\populoss{\hat{f}^{\bW,\bW_0}_{M}} - &L_D(\fwz_{\bv_L}) \\
    \lesssim &~    \sqrt{B(\fhard, L)} \Big[ \sqrt{\frac{{\log (dM) \log (nNM)} +  \log (\delta^{-1})}{{n}}} + \sqrt{\frac{(d^2 + \log M) \delta^{-1}}{M}} \Big].
\end{align*}
Therefore, we have that
\begin{align}
    &~ \populoss{\hat{f}^{\bW,\bW_0}_{M}} -  L_D(\fhard) \nonumber\\
    = &~ \populoss{\hat{f}^{\bW,\bW_0}_{M}} -L_D(\fwz_{\bv_L})+L_D(\fwz_{\bv_L})-  L_D(\fhard) \nonumber\\
    \lesssim &~    \sqrt{B(\fhard, L)} \Big[ \sqrt{\frac{{\log (dM) \log (nNM)} +  \log (\delta^{-1})}{{n}}} + \sqrt{\frac{(d^2 + \log M) \delta^{-1}}{M}} \Big] + \eps_L \linf{\fhard}.
\end{align}
Taking infimum over $L$ proves Theorem \ref{thm::bias sample complexity}.
\end{proof}

\subsection{Proof of Examples in Section \ref{section:bias attention}}\label{sec:proof_examples_bias}
\begin{proposition}[Restatement of Example \ref{exmp::low degree polynomials}]
\label{prop:restate exp low degree polynomials biased}
Consider the target function 
\[g_\star = \frac{1}{N} \sum_{i=1}^N \langle \bx_i^{\otimes 3} \otimes \bx_0^{\otimes 2},\bA\rangle.\]
It has norm bound $B(g_\star,L) = \lfro{\bA}^2L^38^{2L}$. So for any $\eta>0$, if we take $n\ge \exp(\exp(\Theta(1/\eta)))$, $L=\Theta((1 + \log\log n)^{-1}\log n)$, and $M = \Theta( d^2 n)$, the excess risk will scale as $\tO(\sqrt{\| \bA \|_{\Fr}^2/n^{1-\eta}})$. 
\end{proposition}

\begin{proof}[Proof of Proposition \ref{prop:restate exp low degree polynomials biased}]
The value of $B(g_\star,L)$ can be obtained by definition and by direct calculation. As for the second part of the proposition, take $L=r\log n$, where $r>0$ is a parameter that is to be chosen to minimize the excess risk. Eq. \eqref{eqn:bias_RF sample complexity line 2} becomes
\begin{align}
     \tO( \lfro{\bA} r^{3/2} n^{3r - 1/2} + n^{-r\log(2/e)-r\log r-r\log\log n}), \label{eqn::bias_RF polynomial sample complexity}
\end{align}
where $\tO$ hides all the logarithm factors and constants of $n,d$, and $M$. 
To minimize the excess risk, we need to make the two terms in \eqref{eqn::bias_RF polynomial sample complexity} have the same scale. So we set 
$3r - 1/2 = -r\log(2/e)-r\log r-r\log\log n.$ Denoting the solution as $r_\star$, we get that
\[
r_\star = \frac{1/2 - r_\star\log r_\star}{\log(2e^2)+\log\log n}.
\]
So this gives $r_\star < (1/2+1/e)(\log(2e^2)+\log\log n)^{-1}$. Assume $r_\star < 1$ (otherwise $3 r_\star>1/2$ and the excess risk is meaningless), then $r_\star > (2\log(2e^2)+2\log\log n)^{-1}$. Therefore, we get that $r_\star = C(1+\log\log n)^{-1}$, with $C$ being a constant. As a result, when choosing $L=r_\star \log n=\Theta((1 + \log\log n)^{-1}\log n)$, the excess risk Eq. \eqref{eqn::bias_RF polynomial sample complexity} scales as 
$ \tO( \lfro{\bA}n^{3r_\star - \frac{1}{2}}) = \tO( \lfro{\bA} n^{C/(1+\log\log n) - \frac{1}{2}}).$
As a result, let $n>\exp(\exp(C/\eta-1))$ where $C$ is a constant, we get an excess risk that scales as $\tO(\sqrt{\lfros{\bA}^2/n^{1-\eta}})$.
This finishes the proof of Proposition \ref{prop:restate exp low degree polynomials biased}.
\end{proof}

\begin{proposition}[Restatement of Example \ref{exp:correlation functions biased}]~
\label{prop:restate exp correlation functinos biased}
Consider the target function 
    \[ \fhard = \frac{1}{N} \sum_{i=1}^N \<\bx_0,\bx_i\>^p \<\bbeta,\bx_i\> ,~~~ \bbeta \in \S^{d-1}.\]
    It has $B(\fhard,L) = (2L+p)^{p+3} 8^{2L+p}$.
    For any $\eta>0$, choosing the same parameters $(n, L, M)$ as Example~\ref{exmp::low degree polynomials}, the excess risk bound scales as $\tO(\sqrt{{ (\log n + p)^{(p+3)}8^p}/{n^{1-\eta}}})$.

    Furthermore, to reach an accuracy of $0.01$, the \BRFA~model requires $n_\star = \tO((8p+48)^{p+3})$, whereas the \RFA~model requires $n_\star = \tO((4d)^p)$. 
\end{proposition}

\begin{proof}[Proof of Proposition \ref{prop:restate exp correlation functinos biased}]
    The value of $B(\fhard,L)$ can be obtained by direct calculation, and we use the same method as the proof of Proposition \ref{prop:restate exp low degree polynomials biased} to get the $\tO(\sqrt{1/n^{1-\eta}})$ excess risk. To reach an accuracy of $0.01$, we can set $L=3$, note that $\eps_L \linf{\fhard} < 0.006$. Therefore we just need to choose $n_\star$ s.t. $n > \tO((6+p)^{p+3}8^{6+p})$, and this gives $n_\star = \tO((8p+48)^{p+3})$ for the \BRFA~model. Direct calculation using Theorem \ref{thm::sample complexity} gives the value of $n_\star$ for the model \RFA.  This finishes the proof of Proposition \ref{prop:restate exp correlation functinos biased}.
\end{proof}

\begin{proposition}[Restatement of Example \ref{exp:general correlation functions biased}]
\label{prop:restate exp general correlation functions biased}
Consider the target function that has the form 
    \[ \fhard = \frac{1}{N} \sum_{i=1}^N \cos(\<\bx_0,\bx_i\>) \<\bx_i^{\otimes 3},\bG\>. \]
    It has $B(\fhard, L) = \Theta((8e)^{2L})$.
    For any $\eta>0$, choosing the same parameters as Example \ref{exmp::low degree polynomials}, the excess risk bound scales as $\tO(\sqrt{1/{n^{1-\eta}}})$.

    Furthermore, to reach an accuracy of $0.01$. The \BRFA~model requires $n_\star = \tO(1)$, whereas the \RFA~model requires $n_\star = \tO(\Poly(d)\exp(\sqrt{d}))$.
\end{proposition}

\begin{proof}[Proof of Proposition \ref{prop:restate exp general correlation functions biased}]
    We use the expansion of the $\cos$ function,
    $$\fhard = \frac{1}{N} \sum_{i=1}^N \sum_{k=0}^\infty \frac{(-1)^k}{(2k)!} \< \bx_0,\bx_i\>^{2k} \<\bx_i^{\otimes 3},\bG\>. $$
    Then use the Eq. \eqref{eqn:bias_RF sample complexity line 2}, we get that 
    \[B(\fhard,L) = \sum_{k=0}^\infty \frac{(2L+2k)^{k+3/2}8^{L+k}}{(2k)!} \le \sum_{k} \frac{(2k+3)^{k+3/2} 8^k}{(2k)!} (8e)^L = \Theta((8e)^L).\]
    This gives the formula for $B(\fhard,L)$. To reach an accuracy of 0.01, we set $L=3$. Then $B(\fhard,L) = \Theta(1)$. So \BRFA~needs $n_\star = \tO(1)$. And Theorem \ref{thm::sample complexity} and Example \ref{exp:general correlation functions} show that \RFA~model would require $n_\star = \tO(\Poly(d)\exp(\sqrt{d}))$. This finishes the proof of Proposition \ref{prop:restate exp general correlation functions biased}.
\end{proof}

\section{Further experiments}\label{sec::further experiments}

In addition to Section \ref{sec:experiments}, we perform further simulations to examine our theory upon the approximation power of \RFA~\eqref{eqn:RF_attention}, \BRFA~\eqref{eqn:RF_attention_bias}, and \RFMLP~\eqref{eqn::MLP model}. Besides the target function \eqref{eqn:experiment rf_hard}, we consider two additional target functions of the form 
\begin{align}
 \textstyle   f_{3,p}(\bx_{0:N})  =&~ \textstyle  \<\bbeta,\bx_0\>^p, ~~~~~~~~~~~~~~~~~~~~~~~~ p\in \N, ~~~~~~ \bbeta\in \S^{d-1}, \label{eqn:experiment example 1}\\
 \textstyle    f_{4,\gamma}(\bx_{0:N})=&~ \textstyle  \frac{1}{N}\sum_{i=1}^N \<\bx_0,\bS\bx_i\>^3\<\bbeta,\bx_i\>, ~~~ \bbeta\in \S^{d-1},~~~ \bS=\bZ+\gamma I_d \label{eqn:experiment example 3}.
\end{align}
The target function \eqref{eqn:experiment example 1} is a specific instance of Example \ref{exp:functions of x0}, and the target function \eqref{eqn:experiment example 3} is a specific instance of Example \ref{exp:general correlation functions}. In \eqref{eqn:experiment example 3}, we sample $Z_{ij} \simiid \normal(0, 1/d)$ for $(i,j)\in [d]^2$ and vary $\gamma$ in the experiment. Other experimental settings are the same as in Section \ref{sec:experiments}.

Figure \ref{fig:example 1} demonstrates the effect of sequence length $N$ on the performance of three {\RF} models when fitting the target function $f_{3,2}$ \eqref{eqn:experiment example 1}, which solely depends on $x_0$. We observe that a larger sequence length $N$ leads to a larger separation of the test error between {\RFMLP} and the other two random-feature attention models. This result aligns with our sample complexity analysis detailed in Example \ref{exp:functions of x0}, where the sample complexity bound of {\RFMLP} for learning average of functions of $\bx_i$ is found to be $\cO((N/4)^p)$ times greater than that of {\RFA}. 

Figure \ref{fig:example 3,5} demonstrates the performance comparison between {\RFA}~ and {\BRFA}~ under different choices of the token dimension $d$. As we can see from the left panel of Figure \ref{fig:example 3,5}, in the case of $d=4$, {\RFA} outperforms {\BRFA} for large sample size, which may result from that {\BRFA} has slower convergence rate with respect to the sample size $n$ as we state in Example \ref{exp:correlation functions biased}. In the middle and the right panel (i.e., when $d$ is larger), {\BRFA} exceeds {\RFA}, and the largest separation lies in the case of $d=32$. This result is consistent with our analyses that {\BRFA} saves a \Poly{(d)} factor in the sample complexity for $d \gg p$. 

Figure \ref{fig:example 3} demonstrates that {\BRFA} has no advantage in approximating the target function $f_{4,\gamma}$ for $\gamma=0$ (left panel; i.e., $(S_{ij} \simiid \normal(0, 1/d)$, for $(i,j)\in [d]^2$). However, as $\gamma$ increases, {\BRFA} outperforms {\RFA} and their separation increases with a larger $\gamma$ (middle and right panel). Notice that $\lim_{\gamma \to \infty}\<\bx_0,\bS\bx_i\>/\gamma= \<\bx_0,\bx_i\>$. So this result also conforms to our analysis that {\BRFA} is adept at approximating functions of correlations as in Example \ref{exp:correlation functions biased}.

In conclusion, {\RFA} and {\BRFA} have similar performance in fitting functions without correlation structure, such as \eqref{eqn:experiment rf_easy} and \eqref{eqn:experiment example 1}, and \RFA~ may behave even better in some cases. However, {\BRFA} is presumably more powerful than {\RFA} in approximating functions of correlations.

\begin{figure}
\hspace{-2.4cm}
\includegraphics[width=1.28\linewidth,bb= 0 0 1728 360]{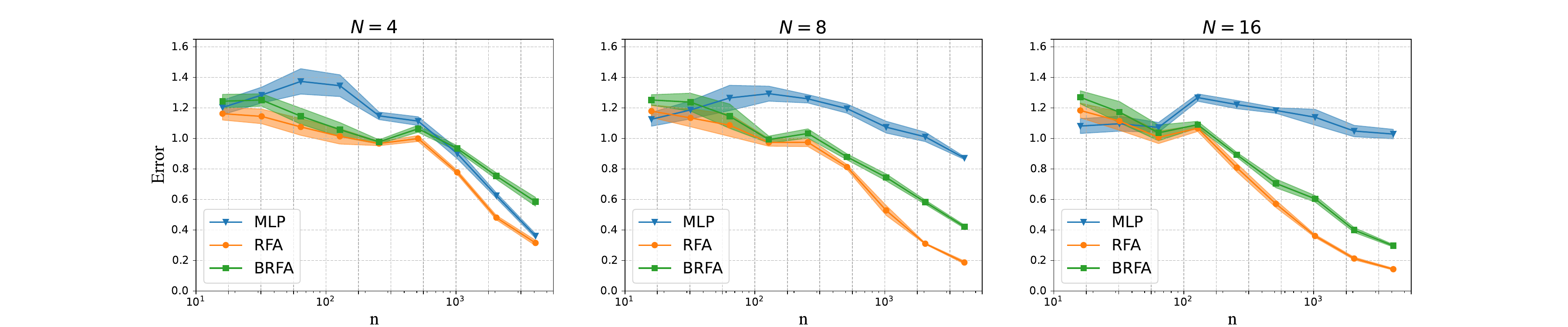}
\vspace{-0.3cm}
\caption{Test error of three \RF~models for learning $f_{3, 2}$ (\ref{eqn:experiment example 1}). We fix $d= 16$ while varying $N = 4$ (left), $8$ (middle), and $16$ (right). The other settings are the same as in Figure \ref{fig:Allfigs}.}

\label{fig:example 1}
\end{figure}

\begin{figure}
\hspace{-2.4cm}
\includegraphics[width=1.28\linewidth,bb= 0 0 1728 360]{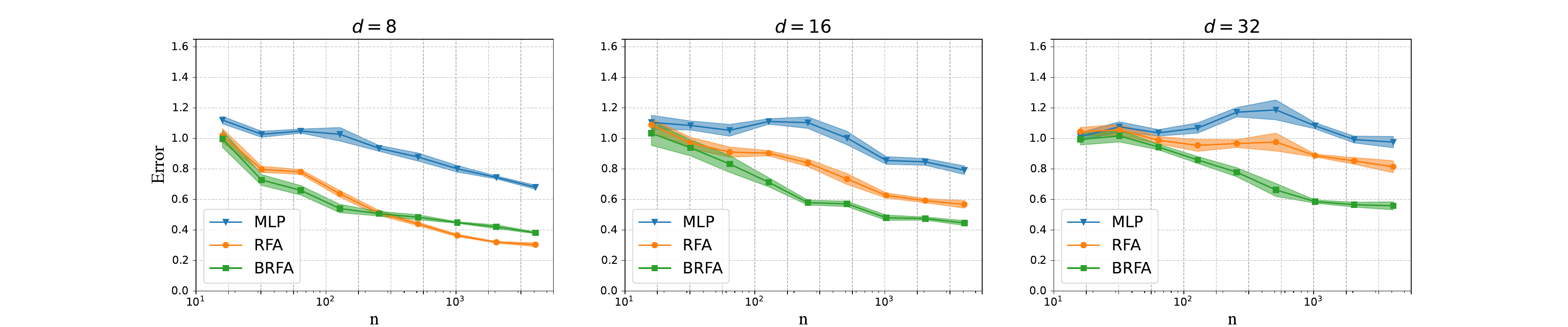}
\vspace{-0.3cm}
\caption{Test error of three \RF~models for learning $f_{2, 3}$ (\ref{eqn:experiment rf_hard}). We fix $N= 16$ while vary $d = 8$ (left), $16$ (middle), and $32$ (right). For a fair comparison among {\RFA}, {\BRFA}, and {\RFMLP}, the number of heads of {\RFMLP} is taken to be $M_{\RFMLP} = 9000$ (left), $17000$ (middle), and $33000$ (right). The other settings are the same as in Figure \ref{fig:Allfigs}.}
\label{fig:example 3,5}
\end{figure}

\begin{figure}
\hspace{-2.4cm}
\includegraphics[width=1.28\linewidth,bb= 0 0 1728 360]{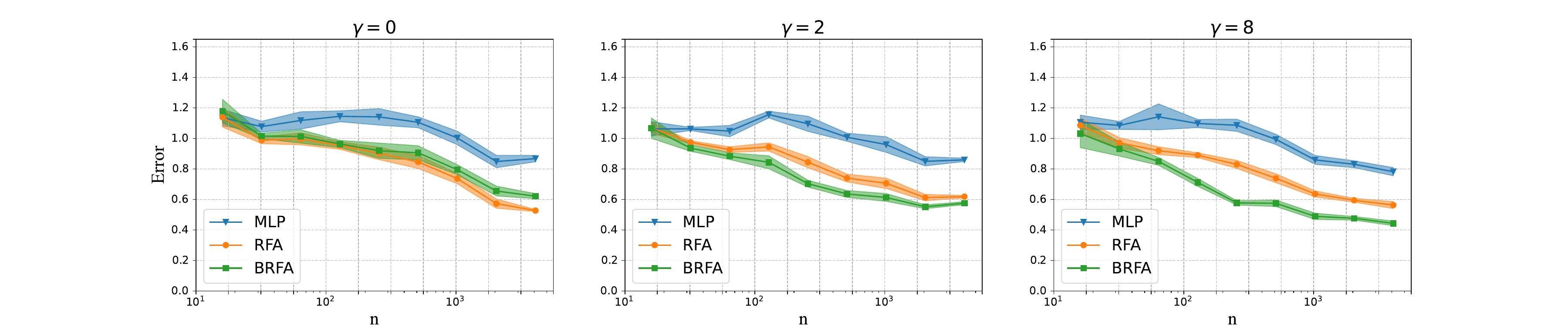}
\vspace{-0.3cm}
\caption{Test error of three \RF~models for learning $f_{4,\gamma}$ (\ref{eqn:experiment example 3}). We choose $\gamma=0$ (left), $2$ (middle), and $8$ (right). The other settings are the same as in Figure \ref{fig:Allfigs}.}

\label{fig:example 3}
\end{figure}

\subsection{Weight matrices in BERT}
\label{app:bert}

As noted in Section \ref{section:bias attention}, we plot the query-key matrices (weight matrices) of the BERT-Base model~\citep{devlin2018bert}\footnote{Downloaded from \url{https://huggingface.co/bert-base-uncased}.} and show that many query-key matrices are diagonally dominated. The BERT-Base model has 12 attention layers with 12 heads in each layer. Denote the query matrix in the $i$-th head of $j$-th layer as $Q_{ij} \in \R^{768\times 64}$ and the key matrix as $K_{ij} \in \R^{768\times 64}$. We compute $W_{ij} = \sqrt{768}\cdot Q_{ij} K_{ij}^\top$ for $i,j$ from 1 to 12, and then take the absolute value for each entry of the weight matrices $W_{ij}$. 
Figure \ref{fig:bert visualization} shows the heat maps of weight matrices of the 2nd, 5th, 8th, and 11th layers, where all matrices are clipped to the top-left $32\times 32$ block. As we can see, a lot of weight matrices are diagonally dominated. We remark that a very recent and concurrent paper \citep{trockman2023mimetic} observed similar phenomena for the ViT-Tiny model. They further show that diagonally dominated weight initialization of self-attention layers allows training transformers faster and obtaining higher final accuracies on CIFAR-10 and ImageNet datasets.

\end{document}